\documentclass[11pt]{article}
\usepackage[a4paper, total={6in, 8in}]{geometry}

\usepackage{times}

\usepackage{amsthm}
\usepackage{amssymb}
\usepackage{enumitem} 
\usepackage{amsmath}

\newtheorem{definition}{Definition}
\newtheorem{theorem}[definition]{Theorem}

\newtheorem{proposition}[definition]{Proposition}
\newtheorem{example}[definition]{Example}
\newtheorem{remark}[definition]{Remark}

\newtheorem*{lemmaA}{Lemma A}
\newtheorem*{lemmaB}{Lemma B}

\usepackage[thinlines]{easytable}
\usepackage{array}
\newcolumntype{P}[1]{>{\centering\arraybackslash}p{#1}} 
\newcolumntype{M}[1]{>{\centering\arraybackslash}m{#1}} 

\usepackage[bookmarks=true]{hyperref}
\hypersetup{%
    colorlinks=true,       
    linkcolor=blue,       
    citecolor=blue,       
    filecolor=black,        
    urlcolor=purple,        
    linktoc=page            
}

\def\layersep{3cm}

\usepackage{upgreek}
\usepackage{csquotes}
\usepackage{colortbl} 

\usepackage{orcidlink}

\usepackage{tikz}
\usetikzlibrary{patterns}
\usetikzlibrary{patterns.meta}

\usepackage[labelfont=bf]{caption}

\usepackage{wrapfig}
\usepackage{graphicx}
\graphicspath{ {images/} }
\usepackage{subcaption}

\usepackage{multirow}

\DeclareRobustCommand{\dotminus}{%
  \mathbin{%
    \ooalign{%
      \hss\raise1ex\hbox{\kern-0.005em.}\hss\cr
      \mathsurround=0pt$-$%
    }%
  }%
}

\makeindex

\title{Machine Learning as\\ Iterated Belief Change \`{a} la Darwiche and Pearl}

\author{Theofanis Aravanis~\orcidlink{0000-0003-0329-3200}}
\date{\small
Department of Digital Systems\\ 
School of Economics and Technology\\
University of the Peloponnese\\ 
Sparta 231 00, Greece\\
\texttt{taravanis@uop.gr} 
}
\usepackage{fancyhdr}
\providecommand{\keywords}[1]{\textbf{\textit{Keywords:}} #1}

\begin{document}

\maketitle
\sloppy

\begin{abstract}
Artificial Neural Networks (ANNs) are powerful machine-learning models capable of capturing intricate non-linear relationships. They are widely used nowadays across numerous scientific and engineering domains, driving advancements in both research and real-world applications. In our recent work \cite{aravanis25a}, we focused on the statics and dynamics of a particular subclass of ANNs, which we refer to as binary ANNs. A binary ANN is a feed-forward network in which both inputs and outputs are restricted to binary values, making it particularly suitable for a variety of practical use cases. Our previous study approached binary ANNs through the lens of belief-change theory, specifically the Alchourr\'{o}n, G\"{a}rdenfors and Makinson (AGM) framework, yielding several key insights. Most notably, we demonstrated that the knowledge embodied in a binary ANN (expressed through its input-output behaviour) can be symbolically represented using a propositional logic language. Moreover, the process of modifying a belief set (through revision or contraction) was mapped onto a gradual transition through a series of intermediate belief sets. Analogously, the training of binary ANNs was conceptualized as a sequence of such belief-set transitions, which we showed can be formalized using full-meet AGM-style belief change. In the present article, we extend this line of investigation by addressing some critical limitations of our previous study. Specifically, we show that Dalal's method for belief change provides a natural basis for a structured, gradual evolution of states of belief. More importantly, given the known shortcomings of full-meet belief change, we demonstrate that the training dynamics of binary ANNs can be more effectively modelled using robust AGM-style change operations ---namely, lexicographic revision and moderate contraction--- that align with the Darwiche-Pearl framework for iterated belief change.
\end{abstract}

\vspace{4mm}


\section{Introduction}

Belief change is the study of how an agent's set of beliefs should be modified ---typically through {\em revision} and/or {\em contraction}--- when confronted with new evidence \cite{gard88,peppas08,ferme18}. The foundations of this field were laid by the influential work of Alchourr\'{o}n, G\"{a}rdenfors and Makinson \cite{agm85}, who introduced a formal model for belief revision and contraction. This model, now known as the {\em AGM framework}, has since become a standard and widely accepted approach for analysing belief change.

Within the AGM framework, a logical theory $K$ (also referred to as {\em belief set}) represents the belief corpus of an agent, while a sentence $\varphi$ represents new evidential input. Belief-change operations are modelled as binary functions, mapping the pair $(K,\varphi)$ to a modified theory, reflecting the result of incorporating or discarding information. The behaviour of these operations is governed by a set of well-accepted rationality postulates.

Since the foundational work of the AGM trio, several concrete belief-change operations have been proposed. One fundamental ---albeit rather trivial--- approach is {\em full-meet} belief change \cite{agm85}. Although this method adheres to the AGM rationality postulates, it has been criticized for producing implausible outcomes \cite{rott00}. A more refined alternative that also respects the AGM postulates is Dalal's belief-change method, which is based on {\em Hamming distance} \cite{dalal88}. Both these schemes will be useful for later analysis.

While the AGM postulates rigorously define the behaviour of {\em single-step} belief change ---that is, the transition from an initial belief set $K$ to its modified form--- they offer little guidance on how to handle {\em sequences} of belief adjustments. This gap is addressed by the theory of {\em iterated} belief change, which focuses on the relationships between successive modifications and how earlier transformations influence subsequent ones \cite{peppas14}. One of the most influential proposals to iterated revision is the work of Darwiche and Pearl
(``DP'' for short), which extends the AGM postulates with additional principles that regulate iterated revisions \cite{darwiche94,darwiche97}. Darwiche and Pearl's approach was later adapted to the context of belief contraction by Konieczny and Pino P\'{e}rez \cite{konieczny17}. Among the most prominent DP-compliant concrete approaches are {\em lexicographic} revision \cite{nayak94,nayak03} and {\em moderate} contraction \cite{ramachandran12}.

In a different vein, one may argue that the process of {\em machine learning} inherently involves a form of belief change. That is, a machine-learning model maintains a body of knowledge (some form of ``beliefs''), which it modifies in response to incoming training data. This perspective underpins our recent study \cite{aravanis25a}, which examined the statics and dynamics of a certain class of Artificial Neural Networks (ANNs) that we term {\em binary} ANNs. A binary ANN is a feed-forward ANN with all inputs and outputs consisting of {\em binary values}, and as such, it is well-suited for various practical applications, such as image processing and pattern recognition on datasets akin to the benchmark MNIST dataset \cite{lecun98}. Binary ANNs were extensively studied in \cite{aravanis25a} through the lens of belief-change theory, leading to several key contributions, including the following:

\begin{itemize}
\item The knowledge of binary ANNs (expressed via their input-output relationship) can {\em symbolically} be represented in terms of a {\em propositional logic language}; specifically, by means of a collection of belief sets. 

\item Within the belief-change framework, the transformation from an initial belief set to a modified one can be conceptualized as a {\em gradual} transition through intermediate states of belief, a perspective that aligns more closely with the cognitive behaviour of real-world agents \cite{lao19,schwitzgebel99}. Analogously, the {\em training} process of binary ANNs, implemented via the fundamental {\em backpropagation} algorithm \cite{rumelhart86}, may be interpreted as a sequence of {\em incremental transformations} between belief sets. Each step in this sequence can be quantitatively characterized by a {\em distance metric} that captures the degree of symbolic divergence between successive belief sets.

\item The aforementioned successive transitions of belief sets can be modelled by means of {\em full-meet} belief change.
\end{itemize}

In this article, we push further towards this line of research, establishing the following contributions:

\begin{itemize}
\item We show that Dalal's method for belief change provides a natural basis for constructing gradual sequences of intermediate belief sets, with transitions organized according to a corresponding distance metric.

\item By means of a {\em refinement} of a distance metric between belief sets proposed in \cite{aravanis25a}, we demonstrate that the successive transitions of belief sets, followed by a binary ANN during training, can be modelled through {\em DP-compatible} AGM revision and contraction operations. In particular, we employ lexicographic revision and moderate contraction, which offer a more robust and well-behaved alternative to the full-meet belief change approach of \cite{aravanis25a}. This shift addresses several documented limitations of full-meet belief change and provides a more principled foundation for understanding belief dynamics in binary ANNs.
\end{itemize}

Evidently, this study is situated at the intersection of neuro-symbolic (hybrid) Artificial Intelligence (AI), with a particular emphasis on integrating AGM-style belief change into machine-learning systems. Although this direction remains under-explored, it is related not only to recent neuro-symbolic work \cite{nawaz25}, but also to a learning-theoretic and dynamic-epistemic line of research on belief revision. In particular, Kelly \cite{kelly98} showed that stricter forms of minimal change may substantially restrict learning power, even leading to forms of ``inductive amnesia'', thereby suggesting that learning power should itself constrain the design of concrete revision policies. Along similar lines, Baltag, Gierasimczuk and Smets \cite{baltag19} studied conditioning \cite{Rott89}, lexicographic revision and minimal revision \cite{Boutilier16} as learning methods, showing that conditioning and lexicographic revision can be universal truth-tracking methods under suitable assumptions, whereas minimal revision is not universal in general. Baltag and Smets \cite{baltag11} further investigated convergence to truth under truthful iterated upgrades, establishing positive convergence results for updates and radical upgrades on finite plausibility models, while analogous results fail for conservative upgrades. At the logical level, Baltag {\em et al.} \cite{baltag19b} introduced a dynamic logic for learning theory, with observation modalities and a learning operator, capable of expressing notions such as stable belief and identifiability in the limit. Most recently, Baccini {\em et al.} \cite{baccini25} refined the picture for minimal revision, showing that although it is not universal in general, it is universal on finitely identifiable spaces and on finite negation-closed settings, while also characterizing suitable priors for minimal revision, conditioning and lexicographic upgrade.

Closer to the present work in the machine-learning setting are the contributions of Coste-Marquis and Marquis \cite{marquis21}, Schwind {\em et al.} \cite{schwind23}, and Schwind {\em et al.} \cite{schwind25}. The first two contributions \cite{marquis21,schwind23} approach the incorporation of symbolic background knowledge into Boolean classifiers through the lens of belief revision, identifying the limitations of traditional belief-change operations and proposing specialized editing mechanisms to adapt classifiers with minimal disruption. The most recent work \cite{schwind25} is directly aligned with our study, introducing a learning framework grounded in {\em improvement operators} \cite{konieczny08,konieczny10} --- a generalization of iterated belief change that permits gradual, noise-tolerant updates. Focusing on binary classification, Schwind {\em et al.} show that such improvement-based models perform competitively when benchmarked against standard machine-learning approaches. Our contribution builds upon this line of research by also leveraging iterated belief change, but rather than developing a belief-change-driven learning system, we focus on modelling the training dynamics of binary ANNs using well-established iterated belief-change operators.

The remainder of this article is structured as follows. The next section establishes the formal prerequisites necessary for our discussion. Section~\ref{section_agm} presents the AGM framework, along with two notable belief-change operations, Dalal's method and full-meet contraction. Section~\ref{section_change_gradual} explores belief change as a process involving gradual transitions through intermediate states of belief. Section~\ref{section_iteration} provides an overview of iterated belief change and its guiding principles. The subsequent sections examine the connection between machine learning and belief change. Specifically, Section~\ref{section_binary_ann} introduces the class of binary ANNs, while Section~\ref{section_ml_dp} demonstrates the compatibility of machine learning with the DP approach.

\section{Formal Prerequisites}
\label{section_preliminaries}

This section provides the formal groundwork for the discussion that follows.

\vspace{2mm}

\noindent {\bf Logic Language:} Throughout this study, we shall work with a propositional language $\mathbb{L}$, built over {\em finitely many} propositional variables (atoms), using the standard Boolean connectives $\wedge$ (conjunction), $\vee$ (disjunction), $\rightarrow$ (implication), $\leftrightarrow$ (equivalence), $\neg$ (negation), and governed by {\em classical propositional logic}. The finite, non-empty set of all propositional variables is denoted by $\mathcal{P}$. The classical inference relation is denoted by $\models$. The symbol $\top$ denotes an arbitrary tautological sentence of $\mathbb{L}$.

\vspace{2mm}

\noindent {\bf Belief Sets:} For a set of sentences $\Gamma$ of $\mathbb{L}$, $Cn(\Gamma)$ denotes the set of all logical consequences of $\Gamma$; i.e., $Cn(\Gamma) = \big\{\varphi\in\mathbb{L}:\Gamma\models\varphi\big\}$. For sentences $\varphi_{1},\ldots,\varphi_{n}$ of $\mathbb{L}$, we shall write $Cn(\varphi_{1},\ldots,\varphi_{n})$ as an abbreviation of $Cn\big(\{\varphi_{1},\ldots,\varphi_{n}\}\big)$. For any two sentences $\varphi,\psi$ of $\mathbb{L}$, we write $\varphi\equiv\psi$ iff $Cn(\varphi) = Cn(\psi)$. An agent's set of beliefs will be modelled by a {\em theory}, also referred to as a {\em belief set}. A theory $K$ is a set of sentences of $\mathbb{L}$ closed under logical consequence; that is, $K = Cn(K)$. As we shall subsequently introduce formal properties of revision and contraction functions, let us first define the simpler operation of {\em expansion}. Accordingly, for a theory $K$ and a sentence $\varphi$ of $\mathbb{L}$, the expansion of $K$ by $\varphi$, denoted by $K + \varphi$, is defined as $K + \varphi = Cn\big(K \cup \{\varphi\}\big)$.

\vspace{2mm}

\noindent {\bf Possible Worlds:} A {\em literal} is a propositional variable $p \in \mathcal{P}$ or its complement (negation). For a finite set of literals $Q$, $|Q|$ denotes the cardinality of $Q$. A {\em possible world} (also named as world, model or interpretation) $r$ is a consistent set of literals, such that, for any propositional variable $p\in\mathcal{P}$, either $p\in r$ or $\neg p\in r$. For a propositional variable $p$ and a world $r$, $p \in r$ means that $p$ is assigned \texttt{true} in $r$, whereas, $\neg p \in r$ means that $p$ is assigned \texttt{false} in $r$. The set of all possible worlds is denoted by $\mathbb{M}$. For a sentence or set of sentences $\varphi$ of $\mathbb{L}$, $[\varphi]$ is the set of worlds implying $\varphi$; that is, $[\varphi] = \big\{ r\in\mathbb{M}:r\models\varphi\big\}$. For the sake of readability, possible worlds will sometimes be represented as sequences (rather than sets) of literals, and the complement of a propositional variable $p$ will be represented as $\overline{p}$, instead of $\neg p$.

\vspace{2mm}

\noindent {\bf Preorders:} A {\em preorder} over a set of possible worlds $M$ is any reflexive and transitive binary relation in $M$. A preorder $\preceq$ is called {\em total} iff, for any $r,r' \in M$, $r \preceq r'$ or $r' \preceq r$. The strict part of $\preceq$ is denoted by $\prec$; i.e., $r \prec r'$ iff $r \preceq r'$ and $r' \npreceq r$. The indifference part of $\preceq$ is denoted by $\approx$; i.e., $r \approx r'$ iff $r \preceq r'$ and $r' \preceq r$. Lastly, $\min(M,\preceq)$ denotes the set of all $\preceq$-minimal possible worlds of $M$; i.e., \mbox{$\min(M,\preceq) = \Big\{r \in M:$ for all $r' \in M$, $r' \preceq r$ entails $r \preceq r' \Big\}$}. When $M$ contains numbers, we simply write $\min(M)$ to denote the minimum number in $M$. 

\vspace{2mm}

\noindent {\bf Boolean Functions:} A ($n$-ary) {\em Boolean function} $f$ is a function that maps every possible combination of $n$ input binary variables to a single binary output ($0$ or $1$); in symbols, $f:\{0,1\}^n \mapsto \{0,1\}$. An example of a ($2$-ary) Boolean function is a Boolean function $f$ that implements the logical operation \texttt{AND}, according to which $f(0,0) = 0$, $f(0,1) = 0$, $f(1,0) = 0$, and $f(1,1) = 1$.

\vspace{2mm}

\noindent {\bf Artificial Neural Networks:} A {\em feed-forward Artificial Neural Network} (ANN) is a computational model that can be formally specified through a directed acyclic graph. The roots of the graph are the inputs $X_1,\dots,X_n$ of the ANN, which form its {\em input layer}. The leaves of the graph are the outputs $y_1,\ldots,y_m$ of the ANN, which form its {\em output layer}. Each node of the graph is called an {\em artificial neuron} or simply {\em neuron}. Neurons that do not belong to the input/output layers of the ANN pertain to the {\em hidden} layer(s) of the ANN. The topology of a representative feed-forward ANN, with a single hidden layer, is depicted in Figure~\ref{fig_ann_topology}.

Each neuron in a feed-forward ANN receives inputs either from the input layer or from the outputs of neurons in a preceding layer. It computes a {\em weighted sum} of these inputs, where each input is multiplied by a {\em weight} specific to the corresponding connection. The neuron then applies a {\em bias} term, adding a constant to the weighted inputs. The resulting sum is then passed through a non-linear {\em activation function} $\sigma$, such as a sigmoid, Rectified Linear Unit (ReLU), or softmax function. Thus, the output $z$ of each neuron can be expressed as

$$z = \sigma\bigg( \sum_{i} w_i\cdot x_i + b \bigg),$$

\noindent where $w_i$'s represents the weight coefficients, $b$ is the bias, and $x_i$'s are the input values from a previous layer. As we will explore in Subsection~\ref{subsection_backprop}, {\em training} an ANN entails repeatedly adjusting its parameters (i.e., the weights $w_i$'s and biases $b$ associated with each neuron) to reduce the difference between the network's predicted outputs and the actual or desired outputs, thereby enhancing its accuracy. 

\begin{figure}[t]
\centering
\begin{tikzpicture}[shorten >=1pt,->,draw=black!50, node distance=\layersep, >=stealth]
    \tikzstyle{every pin edge}=[<-,shorten <=1pt]
    \tikzstyle{neuron}=[circle,fill=black!25,minimum size=17pt,inner sep=0pt]
    \tikzstyle{input neuron}=[neuron, fill=gray!60];
    \tikzstyle{output neuron}=[neuron, fill=gray!60];
    \tikzstyle{hidden neuron}=[neuron, fill=gray!25];
    \tikzstyle{annot} = [text width=7em, text centered]

    \foreach \name / \y in {1,...,4}
        \node[input neuron, pin=left: $X_\y$] (I-\name) at (0,-\y) {};
    \node[] (I-5) at (0,-4.65 cm) {$\vdots$};
    \node[input neuron, pin=left: $X_n$] (I-6) at (0,-5.5 cm) {};
    
    \foreach \name / \y in {1,...,5}
        \path[yshift=0.5cm]
            node[hidden neuron] (H-\name) at (\layersep,-\y cm) {};
    \node[] (H-6) at (\layersep,-5.15 cm) {$\vdots$};
    \node[hidden neuron] (H-7) at (\layersep,-6 cm) {};           
          
    \foreach \name / \y in {1,...,3}   
       \path[yshift=-0.6cm]    
        node[output neuron, pin={[pin edge={->}]right: $y_\y$}] (O-\name) at (6,-\y cm) {};    
        \node[] (O-4) at (6,-4.22 cm) {$\vdots$};
        \node[output neuron, pin={[pin edge={->}]right: $y_m$}] (O-5) at (6,-5 cm) {};
     
    \foreach \source in {1,...,4,6}
        \foreach \dest in {1,...,5,7}
            \path (I-\source) edge (H-\dest);
    
    \foreach \source in {1,...,5,7}
        \foreach \dest in {1,...,3,5}
            \path (H-\source) edge (O-\dest);

    \node[annot,above of=H-1, node distance=0.6cm] (hl) {\footnotesize Hidden Layer};
    \node[annot,left of=hl] {\footnotesize Input Layer};
    \node[annot,right of=hl] {\footnotesize Output Layer};
\end{tikzpicture}
\caption{A feed-forward ANN with a single hidden layer.}
\label{fig_ann_topology}
\end{figure}
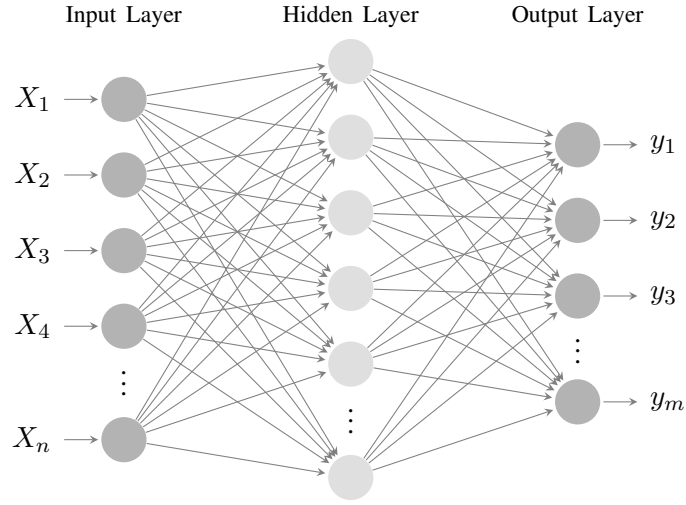

\section{The AGM Framework}
\label{section_agm}

The task of modifying beliefs was formally addressed by Alchourr\'{o}n, G\"{a}rdenfors and Makinson, who introduced a flexible and influential model for belief change, now widely known as the {\em AGM framework} \cite{agm85}. In this framework, an agent's beliefs are captured by a logical theory $K$ (also called a {\em belief set}), while new information (referred to as {\em epistemic input}) is expressed as a logical formula $\varphi$. The AGM framework focuses on two core types of belief modification: {\em belief revision} (or simply, revision) and {\em belief contraction} (or simply, contraction). The AGM trio provided an {\em axiomatic} characterization of both operations using a set of widely accepted {\em rationality postulates}. Later, Katsuno and Mendelzon proposed a characterization of belief revision in terms of possible worlds \cite{katsuno91}. In this section, we outline the AGM framework's axiomatic foundation (Subsection~\ref{subsection_agm_postulates}), the possible-worlds semantics introduced by Katsuno and Mendelzon (Subsection~\ref{subsection_agm_semantic}), and two specific AGM-style belief-change strategies: full-meet belief change \cite{agm85} (Subsection~\ref{subsection_full_meet}) and Dalal's approach \cite{dalal88} (Subsection~\ref{subsection_dalal_operator}).

\subsection{Axiomatic Characterization}
\label{subsection_agm_postulates}

Within the AGM framework \cite{agm85}, the process of belief revision is formalized through a {\em revision function}. A revision function $\ast$ is a binary function that takes as input a belief set $K$ and a formula $\varphi$, and produces a new belief set $K \ast \varphi$, which reflects the outcome of incorporating $\varphi$ into $K$. A revision function $\ast$ is considered an {\em AGM revision function} iff it satisfies postulates \mbox{$(K\ast1)$–$(K\ast8)$}.

\vspace{4mm}

{\renewcommand{\arraystretch}{1.5}
\noindent \begin{tabular}{l l}
$\bf (K\ast1)$ & $K\ast\varphi$ is a theory. \\

$\bf (K\ast2)$ & $\varphi\in K\ast\varphi$. \\

$\bf (K\ast3)$ & $K\ast\varphi \subseteq K+\varphi$. \\

$\bf (K\ast4)$ & If $\neg\varphi\notin K$, then $K + \varphi \subseteq K\ast\varphi$. \\

$\bf (K\ast5)$ & If $\varphi$ is consistent, then $K \ast \varphi$ is also consistent. \\

$\bf (K\ast6)$ & If $\varphi \equiv \psi$, then $K\ast\varphi = K\ast\psi$.\\

$\bf (K\ast7)$ & $K\ast(\varphi\wedge\psi) \subseteq (K\ast\varphi)+\psi$. \\

$\bf (K\ast8)$ & If $\neg\psi\notin K\ast\varphi$, then $(K\ast\varphi)+\psi \subseteq K\ast(\varphi\wedge\psi)$.
\end{tabular}}

\vspace{4mm}

\begin{remark} \label{rem_expansion}
When $\varphi$ is consistent with $K$ (i.e., $\neg\varphi \notin K$), it follows from postulates $(K\ast3)$ and $(K\ast4)$ that revision simplifies to a mere expansion, meaning that $K \ast \varphi = K + \varphi$.
\end{remark}

Similarly, the process of {\em contraction} is captured by a {\em contraction function}. A contraction function $\dotminus$ is a binary function that takes a belief set $K$ and a formula $\varphi$, and returns a new belief set $K \dotminus \varphi$, which represents the result of removing $\varphi$ from $K$. A contraction function $\dotminus$ is classified as an {\em AGM contraction function} iff it adheres to postulates \mbox{$(K\dotminus 1)$–$(K\dotminus 8)$} \cite{agm85}.

\vspace{4mm}

{\renewcommand{\arraystretch}{1.5}
\noindent \begin{tabular}{l l}
$\bf (K\dotminus 1)$ & $K\dotminus\varphi$ is a theory. \\

$\bf (K\dotminus 2)$ & $K\dotminus\varphi \subseteq K$. \\

$\bf (K\dotminus 3)$ & If $\varphi \notin K$, then $K\dotminus \varphi = K$. \\

$\bf (K\dotminus 4)$ & If $\varphi$ is not tautological, then $\varphi \notin K\dotminus \varphi$. \\

$\bf (K\dotminus 5)$ & If $\varphi \in K$, then $K \subseteq (K\dotminus \varphi) + \varphi$. \\

$\bf (K\dotminus 6)$ & If $\varphi\equiv\psi$, then $K\dotminus\varphi = K\dotminus\psi$.\\

$\bf (K\dotminus 7)$ & $(K\dotminus\varphi) \cap (K\dotminus\psi) \subseteq K\dotminus(\varphi \wedge \psi)$. \\

$\bf (K\dotminus 8)$ & If $\varphi \notin K\dotminus(\varphi \wedge \psi)$, then $K\dotminus(\varphi \wedge \psi) \subseteq K\dotminus\varphi$.
\end{tabular}}

\vspace{4mm}

\begin{remark} \label{rem_tautology}
When $\varphi$ is a tautological sentence (i.e., $\varphi \equiv \top$), it follows from postulates $(K\dotminus2)$ and $(K\dotminus5)$ that $K\dotminus\varphi = K$.
\end{remark}

The core idea underlying postulates \mbox{$(K\ast1)$–$(K\ast8)$} and \mbox{$(K\dotminus 1)$–$(K\dotminus 8)$} is the principle of {\em minimal change}, which asserts that the belief set $K$ should be modified as little as possible in response to new epistemic input $\varphi$. For an in-depth discussion of these postulates, the reader is referred to G\"{a}rdenfors \cite[Chapter~3]{gard88} and Peppas \mbox{\cite[Section~8.3]{peppas08}}.

\subsection{Semantic Characterization}
\label{subsection_agm_semantic}

Postulates \mbox{$(K\ast 1)$–$(K\ast 8)$} do {\em not uniquely} determine the revised belief set $K \ast \varphi$ based solely on $K$ and $\varphi$; rather, they delineate the range of all rationally acceptable ways to carry out belief revision. To precisely define the revised set $K \ast \varphi$, one must turn to {\em constructive models} of belief change --- extra-logical mechanisms that encode specific strategies for modifying beliefs. A widely adopted approach of this kind is the model introduced by Katsuno and Mendelzon \cite{katsuno91}, which relies on a particular class of total preorders over possible worlds known as {\em faithful preorders}.

\begin{definition}[Faithful Preorder, \cite{katsuno91}] \label{def_faithfulness}
A total preorder $\preceq_K$ over $\mathbb{M}$ is faithful to a belief set $K$ iff $[K] \neq \varnothing$ entails \mbox{$[K] = \min(\mathbb{M},\preceq_K)$}.
\end{definition}

Intuitively, a faithful preorder $\preceq_{K}$ over the set of possible worlds $\mathbb{M}$ captures the {\em relative plausibility} of those worlds with respect to the belief set $K$. In this ordering, worlds that are considered more plausible given $K$ are ranked lower.

Based on the notion of a faithful preorder, Katsuno and Mendelzon prove the next important result.

\begin{theorem}[\cite{katsuno91}] \label{thm_represent_revision}
A revision function $\ast$ satisfies postulates \mbox{$(K\ast1)$--$(K\ast8)$} iff, for any theory $K$, there exists a total preorder $\preceq_K$ over $\mathbb{M}$, faithful to $K$, such that, for any sentence $\varphi\in\mathbb{L}$:

\begin{center}
\begin{tabular}{l l}
{\bf (R)} & $[K\ast\varphi] = \min([\varphi],\preceq_K)$.
\end{tabular}
\end{center}
\end{theorem}

Therefore, under condition (R), the revised belief set $K \ast \varphi$ is defined in terms of the most plausible $\varphi$-worlds, with respect to $K$.

Combining condition (R) with the well-known Levi and Harper Identities \cite{levi80,harper77}, a possible-worlds characterization for belief contraction can be deduced as well \cite[Section~7]{caridroit17}.

\begin{theorem}[\cite{caridroit17}] \label{thm_represent_contraction}
A contraction function $\dotminus$ satisfies postulates \mbox{$(K\dotminus1)$--$(K\dotminus8)$} iff, for any theory $K$, there exists a total preorder $\preceq_K$ over $\mathbb{M}$, faithful to $K$, such that, for any sentence $\varphi\in\mathbb{L}$:

\begin{center}
\begin{tabular}{l l}
{\bf (C)} & $[K\dotminus\varphi] = [K]\cup \min([\neg\varphi],\preceq_K)$.
\end{tabular}
\end{center}
\end{theorem}

Condition (C) indicates that the contracted belief set $K\dotminus\varphi$ is specified in terms of the set-theoretic union of the $K$-worlds and the most plausible $\neg\varphi$-worlds relative to $K$.

For an AGM revision function $\ast$ and an AGM contraction function $\dotminus$, we shall sometimes write $\preceq^{\ast}$ and $\preceq^{\dotminus}$ to denote faithful preorders over $\mathbb{M}$ corresponding to $\ast$ and $\dotminus$, via conditions (R) and (C), respectively. These superscripts will be omitted whenever the context makes the intended operator clear.

\subsection{Full-Meet Belief Change}
\label{subsection_full_meet}

A particular form of AGM-style belief change is known as {\em full-meet} revision and contraction, which was originally introduced in the foundational work by Alchourr\'{o}n, G\"{a}rdenfors and Makinson \cite{agm85}.\footnote{The definition of full-meet belief change adopted herein departs from that of \cite{agm85}, but it is nonetheless equivalent.}

\begin{definition}[Full-Meet Belief Change] \label{def_full_meet}
An AGM revision (resp., contraction) function $\ast$ (resp., $\dotminus$) is the full-meet AGM revision (resp., contraction) function iff, for any theory $K$, it is specified ---via (R) (resp., (C))--- through a total preorder $\preceq_K$ over $\mathbb{M}$, such that all worlds $r,r'\in\big(\mathbb{M}\setminus[K]\big)$ are equally plausible; that is, $r \approx_K r'$.
\end{definition}

The next observation follows immediately from Definition~\ref{def_full_meet}.

\begin{remark}
Let $\ast$ be the full-meet AGM revision function, and let $\dotminus$ be the full-meet AGM contraction function. Moreover, let $K$ be a theory, and let $\varphi$ be a sentence of $\mathbb{L}$. If $\neg\varphi\in K$, then $K\ast\varphi = Cn(\varphi)$. If $\varphi\in K$, then $K\dotminus\varphi = K\cap Cn(\neg\varphi)$.
\end{remark}

The above remark highlights that, although full-meet revision perfectly respects postulates \mbox{$(K\ast1)$--$(K\ast8)$}, it exhibits a radical behaviour by entirely discarding {\em all} prior beliefs when they conflict with new evidence. For this reason, Rott refers to this type of revision as {\em amnesic} \cite{rott00}. Similarly, full-meet contraction conforms to postulates \mbox{$(K\dotminus1)$--$(K\dotminus8)$} and has been defended by Hansson \cite{hansson06}, but it always produces the {\em smallest} theory ---removing too much from the initial belief set--- compared to other AGM-style contractions. As will be discussed in Section~\ref{section_iteration} (Proposition~\ref{prop_full_meet_dp}), full-meet belief change also performs poorly with respect to iterated belief revision \cite{darwiche97}.

\subsection{Dalal's Revision Operator}
\label{subsection_dalal_operator}

A more fine-grained alternative to the full-meet approach to belief change is Dalal's method \cite{dalal88}. Given a belief set $K$, Dalal assigns plausibility to possible worlds relative to $K$ using a total preorder $\preceq_K$ over $\mathbb{M}$, which is faithful to $K$ and defined according to a {\em Hamming-distance} metric between worlds. In the special case where $K$ is inconsistent (i.e., $[K] = \varnothing$), the revised belief set $K \ast \varphi$ is simply defined as the logical closure of $\varphi$, that is, $Cn(\varphi)$. In the more typical case where $K$ is consistent (i.e., $[K] \neq \varnothing$), Dalal introduces the following formal definitions.

\noindent \begin{definition}[Difference between Worlds] \label{def_diff}
The difference \mbox{Diff $(r,r')$} between two worlds $r$, $r'$ of $\mathbb{M}$ is the set of propositional variables that have different truth values in the two worlds. That is, 

\begin{center}
Diff $(r,r') = \Big(\big(r \setminus r'\big) \cup \big(r' \setminus r\big)\Big) \cap \mathcal{P}$.
\end{center}
\end{definition}

\noindent \begin{definition}[Distance between Belief Sets and Worlds, \cite{dalal88}] \label{def_dist}
The distance D $(K,r)$ between a consistent belief set $K$ and a world $r$ is the cardinality-minimum difference between $r$ and any $K$-world. That is,

\begin{center}
D $(K,r) = \min\bigg(\Big\{ \big|$Diff $(w,r)\big| : w\in[K] \Big\}\bigg)$.
\end{center}
\end{definition}

\noindent \begin{definition}[Dalal's Preorder, \cite{dalal88}] \label{def_dalal_preorder}
Let $K$ be a consistent belief set. Dalal's preorder $\sqsubseteq_{K}$ over $\mathbb{M}$ is the preorder uniquely specified such that, for any \mbox{$r,r'\in\mathbb{M}$},

\begin{center}
\begin{tabular}{l l l}
$r \sqsubseteq_{K} r'$ & iff & D $(K,r)$ $\leqslant$ D $(K,r')$.
\end{tabular}
\end{center}
\end{definition}

\noindent \begin{definition}[Dalal's Revision Operator, \cite{dalal88}] \label{def_dalal}
Dalal's revision operator $\star$ is the revision function induced, via condition (R), by the family of Dalal's preorders $\{\sqsubseteq_{K}\}_{\forall\, K}$ over $\mathbb{M}$.
\end{definition}

Katsuno and Mendelzon pointed out that, for every (consistent) belief set $K$, Dalal's preorder $\sqsubseteq_{K}$ is total and faithful to $K$ \cite[p. 269]{katsuno91}. Consequently, in view of Theorem~\ref{thm_represent_revision}, Dalal's revision operator satisfies postulates \mbox{$(K\ast1)$--$(K\ast8)$}, and thus qualifies as an AGM revision function.

\section{Belief Change as a Gradual Transition of Beliefs}
\label{section_change_gradual}

As previously outlined, the AGM framework provides a set of rationality postulates governing how an agent should transition from an initial belief set $K_1$ to a revised or contracted belief set $K_2$. However, it remains silent on the inner workings of the belief-change process itself --- it focuses {\em solely} on the initial and final states of belief, leaving the intermediate steps unaccounted for.

Motivated by this limitation, and acknowledging that humans typically resist sudden belief alterations due to cognitive and psychological constraints \cite{lao19,kahneman13,schwitzgebel99,vosniadou94}, our previous work \cite{aravanis25a} formally investigated the internal structure of the $K_1$-to-$K_2$ transition. Specifically, we proposed a model in which belief change unfolds as a {\em gradual}, stepwise process, represented by a sequence of intermediate belief sets $H_1,\ldots,H_n$. These sets capture the agent's evolving cognitive state during the transition from $K_1$ to $K_2$, and we denote this progression compactly by \mbox{$\langle K_1, K_2 \rangle = (H_1, \ldots, H_n)$}.

The following belief-revision scenario illustrates the issue.

\begin{example}[Stepwise Revision of Scientific Hypotheses]
Dr. Athena is a climate scientist studying the relationship between solar activity and regional temperature changes. Initially, based on earlier studies, she firmly believes that neither solar flares nor sunspot cycles influence regional climate. Let $a$ denote the proposition ``solar flares do not affect climate'', and $b$ the proposition ``sunspot cycles do not affect climate''. Therefore, her initial belief set is $K_1 = Cn(a \wedge b)$.

Later, Athena receives new empirical data from a long-term satellite observation program, which suggests that both solar flares and sunspot cycles do, in fact, have a measurable effect on temperature patterns. This new information is expressed by the sentence $\varphi = \neg a \wedge \neg b$. Rather than immediately revising her beliefs to adopt $K_2 = Cn(\neg a \wedge \neg b)$, Athena proceeds cautiously. Given the complexity of the phenomena and her commitment to scientific rigour, she considers and evaluates the new claims separately. Her belief revision unfolds through the following intermediate belief set:

\begin{itemize}
\item First, she entertains the possibility that only sunspot cycles are influential, reaching the intermediate belief set $H_1 = Cn(a \wedge \neg b)$.

\item After conducting additional analyses, she begins to suspect that solar flares may also play a role, leading her to adopt the final revised belief set $K_2 = Cn(\neg a \wedge \neg b)$.
\end{itemize}

This progression ---from complete scepticism to full acceptance--- models a gradual epistemic shift driven by evidence and critical analysis. It underscores the idea that belief change, particularly in scientific inquiry, may occur through cautious intermediate steps, rather than through abrupt, wholesale revision.
\end{example}

In \cite{aravanis25a}, we emphasized that, in a formal setting, the sequence of intermediate belief sets $H_1,\ldots,H_n$ may be determined by several principled criteria. For instance, it may be guided by the structure of the faithful preorder $\preceq_{K_1}$ associated with the agent's initial beliefs, or alternatively, by a suitable {\em distance metric} that quantifies (dis)similarity between belief sets. In the present section, we revisit the distance-based approach introduced in \cite{aravanis25a} and extend these metrics (Subsection~\ref{subsection_distance_sets}), and show that Dalal's method for belief change can naturally induce a gradual, well-structured progression through intermediate belief sets (Subsection~\ref{subsection_transition_dalal}).

\subsection{Distance between Belief Sets}
\label{subsection_distance_sets}

A {\em distance metric} between two belief sets $K_1$ and $K_2$ serves to formally capture the degree of difference between the bodies of knowledge they encode. In \cite{aravanis25a}, we introduced two intuitive notions of such distance, referred to as {\em type-A} and {\em type-B} distances, each offering a distinct perspective on how belief sets may diverge. These are defined below.

\begin{definition}[Type-A Distance, \cite{aravanis25a}] \label{def_dist_a}
Let $K_1$, $K_2$ be two consistent belief sets. The type-A distance \mbox{Dist$_{A}(K_1,K_2)$} between $K_1$ and $K_2$ is the cardinality-minimum difference between the $K_1$-worlds and the $K_2$-worlds. That is,

\begin{center}
Dist$_{A}\big(K_1,K_2\big) = \min\bigg(\Big\{ \big|$Diff $(w,w')\big| : w\in[K_1]\ $ and $\ w'\in[K_2] \Big\}\bigg)$.
\end{center}
\end{definition}

\begin{definition}[Type-B Distance, \cite{aravanis25a}] \label{def_dist_b}
Let $K_1$, $K_2$ be two belief sets. The type-B distance \mbox{Dist$_{B}(K_1,K_2)$} between $K_1$ and $K_2$ is the symmetric difference between $[K_1]$ and $[K_2]$. That is,

\begin{center}
Dist$_{B}\big(K_1,K_2\big) = \Big([K_1] \setminus [K_2]\Big) \cup \Big([K_2] \setminus [K_1]\Big)$.
\end{center}
\end{definition}

Observe that type-A distance is a Hamming-based metric representing a non-negative integer, whereas type-B distance corresponds to a set of possible worlds. Example~\ref{ex_distances} concretely illustrates the usage of both type-A and type-B distances.

\begin{example}[Distances between Belief Sets] \label{ex_distances}
Let $\mathcal{P} = \{a,b\}$, and let $K_1 = Cn(\neg a\wedge b)$ and \mbox{$K_2 = Cn(a\leftrightarrow b)$} be two belief sets. Obviously then, \mbox{$[K_1] = \{ \overline{a}b \}$} and $[K_2] = \{ab,\overline{a}\overline{b}\}$. In view of Definitions~\ref{def_dist_a} and~\ref{def_dist_b}, we conclude that Dist$_{A}\big(K_1,K_2\big) = 1$ and \mbox{Dist$_{B}\big(K_1,K_2\big) = \big\{ ab,\overline{a}b,\overline{a}\overline{b}\big\}$}, respectively.
\end{example}

Having formally specified the notion of distance between belief sets, let us consider the sequence $\big(T_1, \ldots, T_l\big) = \big(K_1, H_1, \ldots, H_n, K_2\big)$, representing the full trajectory of belief change from the initial set $K_1$ to the final set $K_2$ through intermediate states $H_1, \ldots, H_n$. Within this sequence, certain natural and meaningful relationships among belief sets can be articulated via the following conditions, for any indices $i,j,m\in\{1,\ldots,l\}$ such that $i < j < m$:

\begin{center}
{\renewcommand{\arraystretch}{2}
\begin{tabular}{l l}
{\bf (DA)} & {\em Dist}$_A\big(T_i,T_m\big) \geqslant$ {\em Dist}$_A\big(T_j,T_m\big)$. \\

{\bf (DB)} & $\big|${\em Dist}$_B\big(T_i,T_m\big)\big| \geqslant\big|${\em Dist}$_B\big(T_j,T_m\big)\big|$. \\

{\bf (SD)} & {\em Dist}$_B\big(T_i,T_m\big) \supseteq$ {\em Dist}$_B\big(T_j,T_m\big)$.
\end{tabular}}
\end{center} 

Conditions (DA) and (DB), originally proposed in \cite{aravanis25a}, are both based on Hamming-type distance metrics, which quantify the number of differing entries between belief sets. In contrast, condition (SD), introduced herein for the first time, constitutes an {\em inclusion-based} refinement of (DB). Rather than measuring only the cardinality of type-B distances, it posits a stronger structural constraint, according to which the type-B distance between $T_i$ and $T_m$ must fully contain the type-B distance between $T_j$ and $T_m$. As such, (SD) captures a more fine-grained and hierarchical view of belief progression, reflecting a deepening convergence toward the final state of belief. Importantly, (SD) will later prove particularly useful in our discussion of ANN dynamics.

\subsection{Gradual Transition of Beliefs via Dalal's Method}
\label{subsection_transition_dalal}

In this subsection, we demonstrate that the belief-change method proposed by Dalal \cite{dalal88} can naturally support gradual and well-ordered transitions through intermediate belief sets in suitable cases.

Dalal's method orders possible worlds according to their Hamming distance from the initial belief set. This ordering naturally stratifies worlds into distance layers around the agent's current beliefs. Such stratification can be used to construct gradual belief transitions, where the agent moves from worlds closest to the initial belief set toward worlds lying at greater Dalal distance. However, this outward movement from the initial belief set does not, in general, guarantee monotonic convergence toward an independently specified final belief set. A monotonic approach toward the final belief set requires additional structural alignment between the Dalal layers around the initial belief set and the distance layers around the final belief set. The following example illustrates a case in which this alignment is present.

\begin{example}(Belief Transition via Dalal) \label{ex_gradual_dalal}
Let $\mathcal{P} = \{a,b,c,d\}$ and assume that \mbox{$K = Cn\big(\neg a\wedge\neg b\wedge\neg c\wedge\neg d\big)$} is the initial belief set of the agent; thus, $[K] = \{\overline{a}\overline{b}\overline{c}\overline{d}\}$. Suppose that the agent modifies her beliefs using Dalal's revision operator $\star$, which assigns (via (R)) at $K$ the (uniquely defined) Dalal's preorder $\sqsubseteq_{K}$ over $\mathbb{M}$:\footnote{$\sqsubset_{K}$ denotes the strict part of $\sqsubseteq_{K}$.} 

\begin{center}
\begin{tabular}{M{1cm} M{1cm} M{1cm} M{1cm} M{1cm} M{1cm} M{1cm} M{1cm} M{1cm}}
$\overline{a}\overline{b}\overline{c}\overline{d}$ 
& $\sqsubset_{K}$ 
& $\overline{a}\overline{b}\overline{c}d$ $\overline{a}\overline{b}c\overline{d}$ $\overline{a}b\overline{c}\overline{d}$ $a\overline{b}\overline{c}\overline{d}$
& $\sqsubset_{K}$ 
& $\overline{a}\overline{b}cd$ $\overline{a}b\overline{c}d$ $\overline{a}bc\overline{d}$ $a\overline{b}\overline{c}d$ $a\overline{b}c\overline{d}$ $ab\overline{c}\overline{d}$
& $\sqsubset_{K}$ 
& $\overline{a}bcd$ $a\overline{b}cd$ $ab\overline{c}d$ $abc\overline{d}$ 
& $\sqsubset_{K}$ 
& $abcd$
\end{tabular}
\end{center}

\noindent Now, let $\varphi = a\wedge b\wedge c\wedge d$ be an epistemic input. Then, the $\star$-revision of $K$ by $\varphi$ produces a belief set $K\star\varphi$, such that $[K\star\varphi] = \min([\varphi],\sqsubseteq_K) = \{ abcd \}$; thus, $K\star\varphi = Cn\big(a\wedge b\wedge c\wedge d\big)$. On that basis, a plausible $\sqsubseteq_{K}$-generated sequence of intermediate belief sets to which the agent adheres during the transition from $K$ to $K\star\varphi$ would be 

\begin{center}
$\big\langle K,K\star\varphi\big\rangle = \big(H_1,H_2,H_3\big)$,
\end{center}

\noindent where $[H_1] = \Big\{ \overline{a}\overline{b}\overline{c}d , \overline{a}\overline{b}c\overline{d} , \overline{a}b\overline{c}\overline{d} , a\overline{b}\overline{c}\overline{d} \Big\}$, $[H_2] = \Big\{ \overline{a}\overline{b}cd , \overline{a}b\overline{c}d , \overline{a}bc\overline{d} , a\overline{b}\overline{c}d , a\overline{b}c\overline{d} , ab\overline{c}\overline{d} \Big\}$, and $[H_3] = \Big\{ \overline{a}bcd , a\overline{b}cd , ab\overline{c}d , abc\overline{d} \Big\}$. Observe that, for any $r\in[K]$, any $r_1\in[H_1]$, any $r_2\in[H_2]$, any $r_3\in[H_3]$, and any $r''\in[K\star\varphi]$, it holds that

\begin{center}
$r\quad \sqsubseteq_K\quad r_1\quad \sqsubseteq_K\quad r_2\quad \sqsubseteq_K\quad r_3\quad \sqsubseteq_K\quad r''$.
\end{center}

Furthermore, letting $\big(T_1,T_2,T_3,T_4,T_5\big) = \big(K,H_1,H_2,H_3,K\star\varphi\big)$, it is true that, for any $i,j,m\in\{1,2,3,4,5\}$ such that $i < j < m$, Dist$_{A}\big(T_i,T_m\big) >$ Dist$_{A}\big(T_j,T_m\big)$. Hence, the agent's belief evolution satisfies condition (DA). This implies that the intermediate belief sets $H_1$, $H_2$, and $H_3$ form genuine in-between states between $K$ and $K\star\varphi$, progressively reducing the type-A distance to the final belief set.
\end{example}

We conclude this section by highlighting that the connections between the notion of gradual beliefs and prominent belief-change frameworks ---such as improvement operators \cite{konieczny08,konieczny10}--- have been explicitly discussed in \cite[Subsection~4.3]{aravanis25a}.

\section{Iterated Belief Change: The DP Approach}
\label{section_iteration}

The AGM framework, as originally introduced by Alchourr\'{o}n, G\"{a}rdenfors and Makinson in \cite{agm85} and presented in Section~\ref{section_agm}, examines only {\em one-step} changes (revisions and/or contractions) and lacks guidelines for {\em iterated} belief change. Perhaps the most influential work addressing the problem of multi-step revision is the approach of Darwiche and Pearl (DP) \cite{darwiche97}. To regulate the process of iterated belief revision, Darwiche and Pearl proposed the subsequent postulates (DP1)--(DP4), where $\ast$ denotes an arbitrary AGM revision function.

\vspace{3mm}

{\renewcommand{\arraystretch}{1.5}
\noindent \begin{tabular}{l l}
{\bf (DP1)} & If $\varphi\models\psi$, then $(K\ast\psi)\ast\varphi = K\ast\varphi$. \\

{\bf (DP2)} & If $\varphi\models\neg\psi$, then $(K\ast\psi)\ast\varphi = K\ast\varphi$. \\
 
{\bf (DP3)} & If $\psi \in K\ast\varphi$, then $\psi \in (K\ast\psi)\ast\varphi$. \\

{\bf (DP4)} & If $\neg\psi \not\in K\ast\varphi$, then $\neg\psi \not\in (K\ast\psi)\ast\varphi$.
\end{tabular}}

\vspace{3mm}

We refer to any AGM revision function that satisfies postulates (DP1)--(DP4) as a {\em DP revision function}. A discussion on postulates (DP1)--(DP4) can be found in the survey of Peppas \cite[Section~7]{peppas14}. Darwiche and Pearl proved that the following constraints (R1)--(R4) correspond to the possible-worlds characterization of postulates (DP1)--(DP4), respectively \cite{darwiche97}.

\vspace{3mm}

{\renewcommand{\arraystretch}{1.5}
\noindent \begin{tabular}{l l}
{\bf (R1)} & If $r, r' \in [\varphi]$, then $r \preceq_{K} r'$ iff $r \preceq_{K\ast\varphi} r'$. \\

{\bf (R2)} & If $r, r' \in [\neg\varphi]$, then $r \preceq_{K} r'$ iff $r \preceq_{K\ast\varphi} r'$. \\
 
{\bf (R3)} & If $r \in [\varphi]$ and $r' \in [\neg\varphi]$, then $r \prec_{K} r'$ implies $r \prec_{K\ast\varphi} r'$. \\

{\bf (R4)} & If $r \in [\varphi]$ and $r' \in [\neg\varphi]$, then $r \preceq_{K} r'$ implies $r \preceq_{K\ast\varphi} r'$.
\end{tabular}}

\vspace{3mm}

In a subsequent work \cite{konieczny17}, Konieczny and Pino P\'{e}rez proposed the following postulates \mbox{(DPC1)--(DPC4)} as counterparts to (DP1)--(DP4) for contraction, where $\dotminus$ is an arbitrary AGM contraction function.

\vspace{3mm}

{\renewcommand{\arraystretch}{1.5}
\noindent \begin{tabular}{l l}
{\bf (DPC1)} & If $\neg\varphi\models\chi$, then $K\dotminus (\varphi \vee \psi) \models K \dotminus \varphi$ iff $(K \dotminus\chi)\dotminus(\varphi \vee \psi) \models (K\dotminus \chi)\dotminus\varphi$. \\

{\bf (DPC2)} & If $\chi\models\varphi$, then $K\dotminus (\varphi \vee \psi) \models K \dotminus \varphi$ iff $(K \dotminus\chi)\dotminus(\varphi \vee \psi) \models (K\dotminus \chi)\dotminus\varphi$.  \\
 
{\bf (DPC3)} & If $\neg\psi\models\chi$, then $(K\dotminus\chi) \dotminus(\varphi\vee\psi) \models (K\dotminus\chi)\dotminus \varphi$ implies $K\dotminus(\varphi \vee \psi)\models K \dotminus\varphi$. \\

{\bf (DPC4)} & If $\chi\models\psi$, then $(K\dotminus\chi) \dotminus(\varphi\vee\psi) \models (K\dotminus\chi)\dotminus \varphi$ implies $K\dotminus(\varphi \vee \psi)\models K \dotminus\varphi$.
\end{tabular}}

\vspace{3mm}

We refer to any AGM contraction function that satisfies postulates (DPC1)--(DPC4) as a {\em DP contraction function}. The possible-worlds characterization of postulates (DPC1)--(DPC4) corresponds to the following constraints (C1)--(C4), respectively.\footnote{Chopra, Ghose, Meyer and Wong also provided an axiomatic characterization of the semantic constraints (C1)--(C4) in terms of postulates that combine revision and contraction operators \cite{chopra08}.}

\vspace{3mm}

{\renewcommand{\arraystretch}{1.5}
\noindent \begin{tabular}{l l}
{\bf (C1)} & If $r, r' \in [\varphi]$, then $r \preceq_{K} r'$ iff $r \preceq_{K\dotminus\varphi} r'$. \\

{\bf (C2)} & If $r, r' \in [\neg\varphi]$, then $r \preceq_{K} r'$ iff $r \preceq_{K\dotminus\varphi} r'$. \\
 
{\bf (C3)} & If $r \in [\neg\varphi]$ and $r' \in [\varphi]$, then $r \prec_{K} r'$ implies $r \prec_{K\dotminus\varphi} r'$. \\

{\bf (C4)} & If $r \in [\neg\varphi]$ and $r' \in [\varphi]$, then $r \preceq_{K} r'$ implies $r \preceq_{K\dotminus\varphi} r'$.
\end{tabular}}

\vspace{3mm}
 
It should be stressed that the symbol $K$ in postulates (DP1)--(DP4) and (DPC1)--(DPC4) denotes a special structure called {\em epistemic state}, rather than a mere belief set. An epistemic state is a richer construct, encompassing not only the associated belief set, but also additional information, such as an ordering over possible worlds encoding their relative plausibility. Although this is an important distinction ---discussed in detail by Darwiche and Pearl \cite{darwiche97} and later by Schwind {\em et al.} \cite{Schwind22}--- it will not affect our exposition, since we shall only work with the possible-worlds characterization of these postulates, namely with conditions (R1)--(R4) and (C1)--(C4).

The following proposition showcases that full-meet belief change is {\em incompatible} with the DP approach.

\begin{proposition} \label{prop_full_meet_dp}
Let $\ast$ be the full-meet AGM revision function. Then, $\ast$ violates the conjunction of postulates (DP1)--(DP4).
\end{proposition}

\begin{proof}
Let $\mathcal{P} = \{a,b\}$ and let $K = Cn(a,b)$ be a theory (i.e., $[K] = \{ab\}$). Since $\ast$ implements full-meet revision, $\ast$ assigns at $K$ (via (R)) a faithful preorder $\preceq_K$ over $\mathbb{M}$, such that 

\begin{center}
$\underbrace{ab}_{[K]} \quad \prec_K\quad a\overline{b}\quad \approx_K\quad \overline{a}b\quad \approx_K\quad \overline{a}\overline{b}$. 
\end{center}

Now, let $\varphi = \neg a\wedge b$ be a sentence of $\mathbb{L}$ (i.e., $[\varphi] = \{\overline{a}b\}$). Then, condition (R) entails that $[K\ast\varphi] = \min([\varphi],\preceq_K) = \{\overline{a}b\}$. Since $\ast$ implements full-meet revision, $\ast$ assigns at $K\ast\varphi$ (via (R)) a faithful preorder $\preceq_{K\ast\varphi}$ over $\mathbb{M}$, such that 

\begin{center}
$\underbrace{\overline{a}b}_{[K\ast\varphi]} \quad \prec_{K\ast\varphi}\quad ab\quad \approx_{K\ast\varphi}\quad a\overline{b}\quad \approx_{K\ast\varphi}\quad \overline{a}\overline{b}$.
\end{center}

\noindent However, through the aforementioned revision process, the transition from the prior preorder $\preceq_{K}$ to the posterior preorder $\preceq_{K\ast\varphi}$ violates postulate (R2), which, given that $ab \prec_{K} a\overline{b}$, demands that $ab \prec_{K\ast\varphi} a\overline{b}$ as well (yet, $ab \approx_{K\ast\varphi} a\overline{b}$). Consequently, $\ast$ violates postulate (DP2), and thus, it violates the conjunction of postulates \mbox{(DP1)--(DP4)}.
\end{proof}

Having outlined the revision and contraction aspects of Darwiche and Pearl's approach, in the remainder of this section we introduce two concrete DP-compliant operations, namely, lexicographic revision and moderate contraction.

\subsection{Lexicographic Revision}
\label{subsection_lex_revision}

Lexicographic revision was introduced by Nayak {\em et al.} as an approach to revision that assigns higher priority to beliefs compatible with the new information over those that are not \cite{nayak94,nayak03}.

\begin{definition}[Lexicographic AGM Revision Function, \cite{nayak94,nayak03}] \label{def_lex}
Let $\ast$ be an AGM revision function, and let $\{\preceq_{K}\}_{\forall\, K}$ be the family of faithful preorders over $\mathbb{M}$, that corresponds to $\ast$ via (R). The revision function $\ast$ is a lexicographic AGM revision function iff $\{\preceq_{K}\}_{\forall\, K}$ satisfies conditions (R1) and (R2), along with the following strict strengthening of (R3) and (R4).

\begin{center}
\begin{tabular}{l l}
{\bf (LR)} & If $r' \in [\neg\varphi]$ and $r \in [\varphi]$, then $r \prec_{K\ast\varphi} r'$.
\end{tabular}
\end{center}
\end{definition}

According to Definition~\ref{def_lex}, lexicographic revision makes every $\varphi$-world strictly more plausible than every $\neg\varphi$-world, while preserving the ordering within each of $[\neg\varphi]$ and $[\varphi]$. Note that the three semantic constraints (R1), (R2) and (LR) are sufficient to {\em uniquely} determine the posterior total preorder $\preceq_{K\ast\varphi}$, which is faithful to the revised belief set $K\ast\varphi$.

\subsection{Moderate Contraction}
\label{subsection_mod_contraction}

Ramachandran {\em et al.} proposed a contraction strategy that can be viewed as an analogue of lexicographic revision \cite{ramachandran12}. This type of contraction, known as {\em moderate contraction} (also referred to as priority contraction), gives precedence to the $\neg\varphi$-worlds over the $\varphi$-worlds, when contracting by $\varphi$.

\begin{definition}[Moderate AGM Contraction Function, \cite{ramachandran12}] \label{def_moderate}
Let $\dotminus$ be an AGM contraction function, and let $\{\preceq_{K}\}_{\forall\, K}$ be the family of faithful preorders over $\mathbb{M}$, that corresponds to $\dotminus$ via (C). The contraction function $\dotminus$ is a moderate AGM contraction function iff $\{\preceq_{K}\}_{\forall\, K}$ satisfies conditions (C1) and (C2), along with the following strict strengthening of (C3) and (C4).

\begin{center}
\begin{tabular}{l l}
{\bf (MC)} & If $r \in [\neg\varphi]$, $r' \in [\varphi]$ and $r'\notin[K\dotminus\varphi]$, then $r \prec_{K\dotminus\varphi} r'$.
\end{tabular}
\end{center}

\noindent In the limiting case where the epistemic input $\varphi\equiv\top$ (i.e., $[\neg\varphi] = \varnothing$), it follows that $\preceq_{K\dotminus\varphi}\ =\ \preceq_K$.
\end{definition}

According to Definition~\ref{def_moderate}, a moderate contraction ensures that all $\varphi$-worlds, except for the most $K$-plausible ones, are ordered after the $\neg\varphi$-worlds. Note that the three semantic constraints (C1), (C2) and (MC) are sufficient to {\em uniquely} determine the posterior total preorder $\preceq_{K\dotminus\varphi}$, which is faithful to the contracted belief set $K\dotminus\varphi$.

The following observation follows immediately from Definitions~\ref{def_lex} and~\ref{def_moderate}.

\begin{remark} \label{rem_subclass_dp}
Any lexicographic AGM revision function satisfies postulates (DP1)--(DP4), and any moderate AGM contraction function satisfies postulates (DPC1)--(DPC4).
\end{remark}

Figure~\ref{fig_dp} provides a graphical representation of both lexicographic revision and moderate contraction. 

\begin{figure}[h!]
\centering
\begin{subfigure}{0.45\textwidth}
\centering
\begin{tikzpicture}
\node[] at (1.5,9.5) {$[\varphi]$};
\node[] at (0.5,9.5) {$[\neg\varphi]$};
\def\w{1}
\def\h{1}

\foreach \i in {0,...,3} {
    \draw[pattern={Lines[angle=45,distance=3pt,line width=0.5pt]}, pattern color=red!50, thick] (0,{(\i+5)*\h}) rectangle ++(\w,\h); }

\foreach \i in {0,...,3} {
    \draw[pattern={Lines[angle=-45,distance=3pt,line width=0.5pt]}, pattern color=green!50, thick] (\w,{(\i+1)*\h}) rectangle ++(\w,\h); }
\end{tikzpicture}
\caption{Lexicographic revision}
\end{subfigure}
\hfill
\begin{subfigure}{0.45\textwidth}
\centering
\begin{tikzpicture}
\node[] at (1.5,8.5) {$[\varphi]$};
\node[] at (0.5,8.5) {$[\neg\varphi]$};
\def\w{1}
\def\h{1}

\foreach \i in {0,...,3} {
    \draw[pattern={Lines[angle=45,distance=3pt,line width=0.5pt]}, pattern color=red!50, thick] (0,{(\i+1)*\h}) rectangle ++(\w,\h); }

\draw[pattern={Lines[angle=-45,distance=3pt,line width=0.5pt]}, pattern color=green!50, thick] (\w,\h) rectangle ++(\w,\h);

\foreach \i in {0,...,2} {
    \draw[pattern={Lines[angle=-45,distance=3pt,line width=0.5pt]}, pattern color=green!50, thick] (\w,{(\i+5)*\h}) rectangle ++(\w,\h); }
\end{tikzpicture}
\caption{Moderate contraction}
\end{subfigure}
\caption{Graphical representation of lexicographic revision (a) and moderate contraction (b). A state of belief is modified (revised or contracted) by a sentence $\varphi$. $\varphi$-worlds are represented by green rectangles, while $\neg\varphi$-worlds are represented by red rectangles. The lower a rectangle is, the more plausible the represented worlds are relative to the state of belief --- thus, the lowest rectangles represent the worlds of the modified belief set.}
\label{fig_dp}
\end{figure}
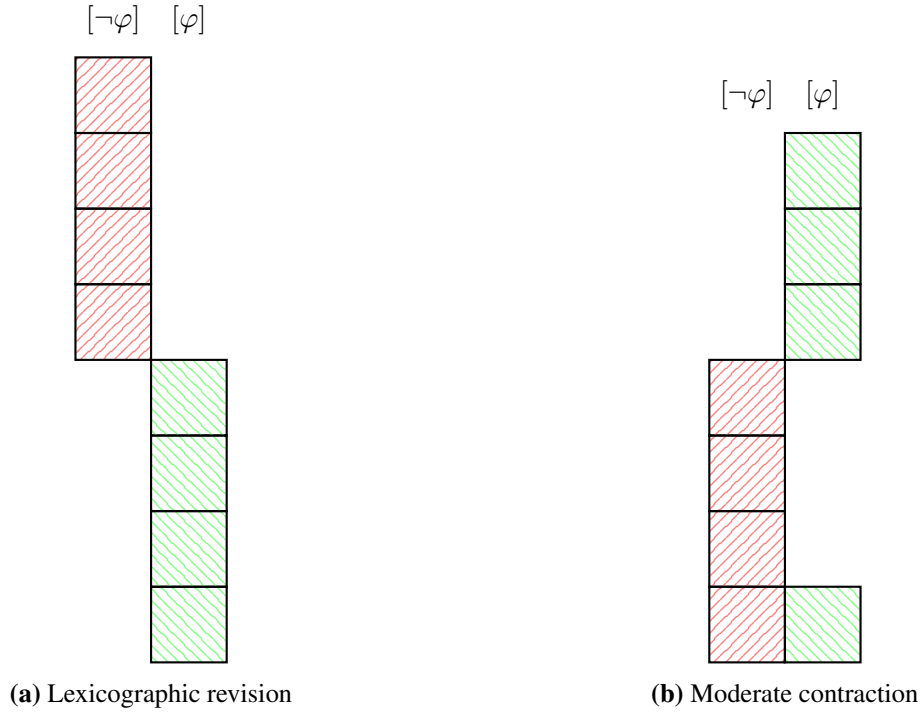

\section{Binary Artificial Neural Networks}
\label{section_binary_ann}

Building on the preceding discussion of belief change, we now investigate its relationship with machine learning. We begin by introducing a special class of Artificial Neural Networks (ANNs), termed {\em binary} ANNs, as proposed in \cite{aravanis25a}. A basic familiarity with the core principles of ANNs is assumed --- readers seeking further background may consult standard textbooks such as \cite{goodfellow16,bishop96,haykin94}.

\subsection{Architecture}
\label{subsection_architecture}

A binary ANN is a feed-forward ANN whose inputs $X_1,\dots,X_n$ take {\em binary values} (either $0$ or $1$).\footnote{Since any input data can be transformed into a binary representation, the binary-input assumption in binary ANNs does not impose a limitation on the type of input data.} Each neuron $i\in\{1,\ldots,m\}$ in the output layer of a binary ANN produces a real-valued output \mbox{$y_i\in[0,1]$} (e.g., via a sigmoid or softmax activation function). Given a real-valued threshold $\tau_i\in[0,1]$, each output $y_i$ of the network identifies a {\em binary output} $Y_i$, such that: 

\begin{center}
$Y_i = 
\left\{ \begin{array}{ll} 1 & \hspace{3mm} \mbox{if} \hspace{4mm} \mbox{$y_i \geqslant \tau_i$}
\\ \\ 
0 & \hspace{3mm} \mbox{if} \hspace{4mm} \mbox{$y_i < \tau_i$} 
\end{array}\right.$
\end{center}

Hence, a binary ANN maps binary inputs $X_1,\dots,X_n$ to binary outputs $Y_1,\dots,Y_m$, while imposing {\em no} constraints on the values of its weights and biases.\footnote{Binary ANNs in this sense {\em differ} from models that employ only binary parameters, such as those described in \cite{hubara16}.} The topology of a binary ANN, with a single hidden layer, is illustrated in Figure~\ref{fig_binary_ann} (cf. Figure~\ref{fig_ann_topology} of Section~\ref{section_preliminaries}). 

Evidently, the minimal assumptions and constraints that characterize binary ANNs render them suitable for a wide range of real-world applications. As illustrative examples, Subsection~\ref{subsection_examples} will demonstrate that a binary ANN can be trained to learn logical operations, as well as to recognize handwritten digits from the well-known MNIST (Modified National Institute of Standards and Technology) dataset \cite{lecun98}.

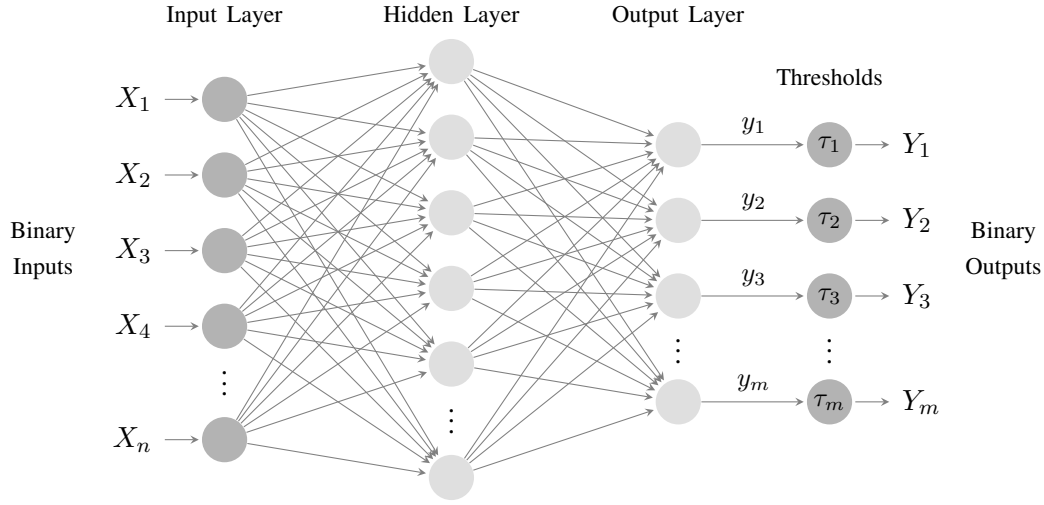
\begin{figure}[t]
\hspace*{-1.2cm}
\centering
\begin{tikzpicture}[shorten >=1pt,->,draw=black!50, node distance=\layersep, >=stealth]
    \tikzstyle{every pin edge}=[<-,shorten <=1pt]
    \tikzstyle{neuron}=[circle,fill=black!25,minimum size=17pt,inner sep=0pt]
    \tikzstyle{input neuron}=[neuron, fill=gray!60];
    \tikzstyle{output neuron}=[neuron, fill=gray!25];
    \tikzstyle{binary output}=[neuron, fill=gray!60];
    \tikzstyle{hidden neuron}=[neuron, fill=gray!25];
    \tikzstyle{annot} = [text width=7em, text centered]

    \foreach \name / \y in {1,...,4}
        \node[input neuron, pin=left: $X_\y$] (I-\name) at (0,-\y) {};
    \node[] (I-5) at (0,-4.65 cm) {$\vdots$};
    \node[input neuron, pin=left: $X_n$] (I-6) at (0,-5.5 cm) {};
    
    \foreach \name / \y in {1,...,5}
        \path[yshift=0.5cm]
            node[hidden neuron] (H-\name) at (\layersep,-\y cm) {};
    \node[] (H-6) at (\layersep,-5.15 cm) {$\vdots$};
    \node[hidden neuron] (H-7) at (\layersep,-6 cm) {};           
          
    \foreach \name / \y in {1,...,3}   
       \path[yshift=-0.6cm]    
        node[output neuron] (O-\name) at (6,-\y cm) {};    
        \node[] (O-4) at (6,-4.22 cm) {$\vdots$};
        \node[output neuron] (O-5) at (6,-5 cm) {};
     
    \foreach \name / \y in {1,...,3}   
       \path[yshift=-0.6cm]    
        node[binary output, pin={[pin edge={->}]right: $Y_\y$}] (OO-\name) at (8,-\y cm) {\small $\tau_\y$};    
        \node[] (OO-4) at (8,-4.22 cm) {$\vdots$};
        \node[binary output, pin={[pin edge={->}]right: $Y_m$}] (OO-5) at (8,-5 cm) {\small $\tau_m$};     
     
    \foreach \source in {1,...,4,6}
        \foreach \dest in {1,...,5,7}
            \path (I-\source) edge (H-\dest);
    
    \foreach \source in {1,...,5,7}
        \foreach \dest in {1,...,3,5}
            \path (H-\source) edge (O-\dest);
    \draw (O-1)--(OO-1) node[midway, above] {\small $y_1$};
    \draw (O-2)--(OO-2) node[midway, above] {\small $y_2$};
    \draw (O-3)--(OO-3) node[midway, above] {\small $y_3$};
    \draw (O-5)--(OO-5) node[midway, above] {\small $y_m$};

    \node[annot,above of=H-1, node distance=0.6cm] (hl) {\footnotesize Hidden Layer};
    \node[annot,left of=hl] {\footnotesize Input Layer};
    \node[annot,right of=hl] {\footnotesize Output Layer};
    \node[annot,above of=OO-1, node distance=0.9cm] (tr) {\footnotesize Thresholds};
    \node[annot] (in) at (-2.4,-3) {\footnotesize Binary\\\footnotesize Inputs};
    \node[annot] (out) at (10.3,-3) {\footnotesize Binary\\\footnotesize Outputs};
\end{tikzpicture}
\caption{A binary ANN with a single hidden layer. The network maps binary inputs $X_1,\dots,X_n$ to binary outputs $Y_1,\dots,Y_m$.}
\label{fig_binary_ann}
\end{figure}

The following crucial observation shows that the operation of an arbitrary binary ANN with {\em multiple} outputs can be analysed in terms of a collection of binary ANNs, each with a {\em single} output.

\begin{remark} \label{rem_single_multiple}
The input-output relationship of any binary ANN with multiple outputs can be analysed in terms of a collection of single-output binary ANNs, each corresponding to one binary output.
\end{remark}

Guided by Remark~\ref{rem_single_multiple}, and to simplify the presentation, we shall restrict our analysis to {\em single-output} binary ANNs.

First, note that a binary ANN with a single binary output $Y$ induces a {\em Boolean function} $f$, which we refer to as the Boolean function {\em of} $Y$. That is,

\begin{center}
$Y = f\big(X_1,\ldots,X_n\big)$.
\end{center}

\noindent The Boolean function $f$ of $Y$ can be represented as a propositional formula $\psi$ of $\mathbb{L}$, in the sense that $f$ and $\psi$ yield {\em identical truth tables}. To illustrate this correspondence, assume that the set of propositional variables $\mathcal{P}$ contains exactly one variable for each input of the ANN; i.e., \mbox{$|\mathcal{P}| = n$}. Moreover, let $w$ be a possible world of $\mathbb{M}$ such that, for any $p_i\in\mathcal{P}$ with $i\in\{1,\ldots,n\}$, $p_i\in w$ whenever $X_i = 1$, and $\neg p_i\in w$ whenever $X_i = 0$. Clearly then, there exists a {\em one-to-one correspondence} between the propositional variables in $\mathcal{P}$ and the binary inputs $X_1,\ldots,X_n$. Accordingly, the following equivalence holds for the binary output $Y$:

\begin{center}
\begin{tabular}{l l l}
$f\big(X_1,\ldots,X_n\big) = 1$ & iff & $w\in[\psi]$.
\end{tabular}
\end{center}

\noindent That is, the possible world $w$ satisfies the formula $\psi$ iff the Boolean function $f$ of $Y$ maps to $1$ the truth values of the propositional variables assigned in $w$. Moreover, since for the belief set \mbox{$K = Cn(\psi)$} it holds that $[K] = [\psi]$, the Boolean function $f$ of the binary output $Y$ can, equivalently, be represented by the belief set $K$.

\subsection{Training through Backpropagation}
\label{subsection_backprop}

In standard ANNs, each neuron produces an output by applying a non-linear activation function to a weighted sum of its inputs,  where the weights and biases are trainable parameters corresponding to connections between successive layers. Training involves \textit{forward propagation}, where outputs are computed sequentially from the input to the output layer using the current parameters, followed by \textit{backpropagation}, which computes the gradient of a {\em loss function} $\mathcal{L}$ ---quantifying prediction error--- with respect to each parameter via the chain rule. These gradients are used in the \textit{parameter-update} step, where weights and biases are adjusted in the direction that minimizes $\mathcal{L}$. Repeating this process over multiple epochs leads to progressive minimization of $\mathcal{L}$ and improved predictive performance \cite{rumelhart86}.

As demonstrated in \cite{aravanis25a}, the notion of intermediate states of belief within the context of belief change can also be meaningfully applied to the training process of binary ANNs. Specifically, for a single-output binary ANN with output $Y$, one can identify a sequence of belief sets $K_1, \ldots, K_n$ such that each pair $K_i$ and $K_{i+1}$ (for \mbox{$i \in \{1, \ldots, n-1\}$}) correspond to the Boolean functions representing $Y$ immediately {\em before} and {\em after} the $i$-th update of the network parameters, respectively. In this view, $K_1$ and $K_n$ denote the belief sets before and after the entire training process, while $K_2$ through $K_{n-1}$ represent the transitional belief sets throughout training. 

Under additional assumptions ensuring that successive thresholded outputs do not introduce new disagreements with the final output behaviour, it can be obtained that, for any \mbox{$i,j\in\{1,\ldots,n\}$} such that $i < j$,

\begin{center}
$\big|${\em Dist}$_B\big(K_i,K_n\big)\big| \geqslant \big|${\em Dist}$_B\big(K_j,K_n\big)\big|$.
\end{center}

This relationship reflects the structure outlined in condition (DB) of Subsection~\ref{subsection_distance_sets}, and establishes a natural type-B distance-based ordering among the belief sets $K_1, \ldots, K_n$ induced by the training process of a single-output binary ANN. Specifically, it implies that, as training progresses, the internal representation of the ANN ---in terms of its logical behaviour--- becomes {\em increasingly aligned} with the final learned configuration, as determined by the training labels (targets).

In a similar vein, one may posit an even {\em stronger} structural relationship among the belief sets $K_1,\ldots,K_n$ by invoking condition (SD) of Subsection~\ref{subsection_distance_sets}, formulated in this work. This condition asserts that, for any $i,j\in\{1,\ldots,n\}$ such that $i < j$,

\begin{center}
{\em Dist}$_B\big(K_i,K_n\big) \supseteq$ {\em Dist}$_B\big(K_j,K_n\big)$.
\end{center}

\noindent Intuitively, this means that the discrepancies (in terms of type-B distance) between earlier belief sets and the final one not only remain greater in magnitude, but also contain the discrepancies of later belief sets as subsets. Clearly, this stronger inclusion-based condition implies the inequality of cardinalities mentioned above, i.e., {\em Dist}$_B\big(K_i,K_n\big) \supseteq$ {\em Dist}$_B\big(K_j,K_n\big)$ entails \mbox{$\big|${\em Dist}$_B\big(K_i,K_n\big)\big| \geqslant \big|${\em Dist}$_B\big(K_j,K_n\big)\big|$}.

\subsection{Illustrative Examples}
\label{subsection_examples}

This aforementioned hierarchical relationship among the belief sets $K_1, \ldots, K_n$ is not merely a theoretical construct, but rather a {\em descriptive regularity} often observed in the training dynamics of practical binary ANNs. While not guaranteed by design, the structure imposed by condition (SD) tends to emerge naturally in many real-world models, often aided by modern training practices such as adaptive optimizers, learning-rate schedules, momentum, and regularization \cite{goodfellow16}. These methods implicitly encourage smoother, more directed convergence, thereby reinforcing the subset-based alignment of belief sets over time. In this subsection, we illustrate this phenomenon through two case studies involving standard binary ANNs, each demonstrating how the training trajectory consistently conforms to the progression prescribed by condition (SD).

\begin{example}[Binary ANN Training for a Logical Operation] \label{ex_ann_boolean}
A binary ANN is implemented using the \texttt{Keras} Python library, with the goal of learning the logical operation ``at least two''. The architecture of the ANN is as follows:

\begin{itemize}
\item Input Layer: $3$ neurons corresponding to the binary inputs $X_1$, $X_2$, and $X_3$.

\item Hidden Layer: One hidden layer with $100$ neurons, each equipped with a ReLU activation function, selected to balance expressive capacity and architectural simplicity.

\item Output Layer: A single neuron with a sigmoid activation function that produces a real-valued estimation $\hat{y}\in[0,1]$. The estimation $\hat{y}$ is passed through a threshold $\tau = 0.5$, resulting in a binary estimation $\hat{Y}$ produced by the binary output $Y$ of the ANN.
\end{itemize}

Since the ANN is intended to implement the logical operation ``at least two'', the corresponding Boolean function $f$ for the binary output $Y$ must satisfy the following mapping:

\begin{center}
$
\begin{aligned}
&f(0,0,0) = 0,\quad f(0,0,1) = 0,\quad f(0,1,0) = 0,\quad f(0,1,1) = 1,\\
&f(1,0,0) = 0,\quad f(1,0,1) = 1,\quad f(1,1,0) = 1,\quad f(1,1,1) = 1.
\end{aligned}
$
\end{center}

This mapping defines the training dataset for the ANN, consisting of $2^3 = 8$ samples. The co-domain of $f$ specifies the binary labels that guide the learning process. The ANN is compiled using the binary cross-entropy loss function, which is standard for binary classification tasks. All network parameters are randomly initialized prior to training.

\begin{table}[t]
\centering\scriptsize
{\renewcommand{\arraystretch}{1.5}
\begin{tabular}{|M{1.8cm}||M{2.4cm}|M{2.4cm}|M{2.4cm}|M{2.4cm}||M{0.7cm}|}
 \hline
 \rowcolor{gray!10}
 {\bf Input} & {\bf 1st Output}\, ($\hat{Y}$)
 & {\bf 2nd Output}\, ($\hat{Y}$) & {\bf 3rd Output}\, ($\hat{Y}$) & {\bf 4th Output}\, ($\hat{Y}$) & {\bf Label} \\
 \hline\hline
 $(0,0,0)$ & $1$ & $0$ & $0$ & $0$ & {\bf 0} \\
 \hline
 $(0,0,1)$ & $1$ & $1$ & $0$ & $0$ & {\bf 0} \\
 \hline
 $(0,1,0)$ & $1$ & $1$ & $1$ & $0$ & {\bf 0} \\
 \hline
 $(0,1,1)$ & $1$ & $1$ & $1$ & $1$ & {\bf 1} \\
 \hline
 $(1,0,0)$ & $1$ & $1$ & $1$ & $0$ & {\bf 0} \\
 \hline
 $(1,0,1)$ & $1$ & $1$ & $1$ & $1$ & {\bf 1} \\
 \hline
 $(1,1,0)$ & $1$ & $1$ & $1$ & $1$ & {\bf 1} \\
 \hline
 $(1,1,1)$ & $1$ & $1$ & $1$ & $1$ & {\bf 1} \\
 \hline\hline
 $[K_i]$ & $\mathbb{M}$ & $\mathbb{M} \setminus \big\{ \overline{a}\overline{b}\overline{c} \big\}$ & $\mathbb{M} \setminus \big\{ \overline{a}\overline{b}\overline{c},\overline{a}\overline{b}c \big\}$ & $\big\{ \overline{a}bc , a\overline{b}c , ab\overline{c} , abc \big\}$ & --- \\
 \hline
 {\em Dist}$_B(K_i,K_4)$ & $\big\{\overline{a}\overline{b}c,a\overline{b}\overline{c},\overline{a}b\overline{c},\overline{a}\overline{b}\overline{c} \big\}$ & $\big\{\overline{a}\overline{b}c,a\overline{b}\overline{c},\overline{a}b\overline{c} \big\}$ & $\big\{ a\overline{b}\overline{c},\overline{a}b\overline{c} \big\}$ & $\varnothing$ & --- \\
 \hline
 {\bf Accuracy} & $50\%$ & $62.5\%$ & $75\%$ & $100\%$ & --- \\
 \hline
\end{tabular}}
\caption{Successive transitions (estimations) of the binary output $Y$ produced by the ANN during training, evaluated across all $8$ possible binary input combinations. For each training stage \mbox{$i\in\{1,2,3,4\}$}, the table reports the corresponding belief set $K_i$, the type-B distance {\em Dist}$_B(K_i,K_4)$ from the final belief set $K_4$, and the classification accuracy at that stage.}
\label{table_keras}
\end{table}

On that basis, Table \ref{table_keras} presents the four successive transitions (estimations) of the binary output $Y$ produced by the ANN during training, evaluated across all $8$ possible binary input combinations. As shown, the network's output increasingly aligns with the ground-truth labels, leading to a strict monotonic increase in accuracy, which eventually reaches $100\%$. This indicates that the network successfully learns the target logical function.

To formalize the internal behaviour of the ANN during training, let us denote the propositional variables corresponding to the binary inputs $X_1$, $X_2$, $X_3$ by $a$, $b$, $c$, respectively. Let $K_1$, $K_2$, $K_3$, $K_4$ denote the belief sets representing the Boolean functions encoded by the ANN's output at successive training stages, as illustrated in Table~\ref{table_keras}. From the table, we observe that $K_1 = Cn(\varnothing)$, $K_2 = Cn( a\vee b\vee c)$, $K_3 = Cn(a\vee b)$, and $K_4 = Cn\Big((a\wedge b)\vee(a\wedge c)\vee(b\wedge c) \Big)$. Hence, the final Boolean function $f$ learned by the ANN is captured by the belief set $K_4$, or equivalently, by the propositional formula $\psi = (a\wedge b)\vee(a\wedge c)\vee(b\wedge c)$, which precisely encodes the intended logical operation ``at least two''.

Lastly, Table~\ref{table_keras} clearly demonstrates that the evolution of the belief sets $K_1$ through $K_4$ satisfies condition (SD), since Dist$_{B}\big(K_1,K_4\big) \supset\, $ Dist$_{B}\big(K_2,K_4\big)\supset\, $ Dist$_{B}\big(K_3,K_4\big)$.
\end{example}

Example~\ref{ex_mnist} concludes this section by demonstrating the training of a binary ANN on the MNIST dataset, a standard benchmark of handwritten digits, where each image is represented by grayscale pixel values \cite{lecun98}.

\begin{example}[Binary ANN Training on the MNIST Dataset] \label{ex_mnist}
A binary ANN is implemented using the \texttt{Keras} Python library, with the goal of distinguishing between handwritten digits $0$ and $1$ from the MNIST dataset. Our training dataset consists of $30$ samples. Each MNIST image is down-sampled from its original resolution of $28\times 28$ pixels to $10\times 10$, and then binarized into black-and-white pixels. As a result, each image is encoded as a sequence of $10\times 10 = 100$ binary values. Figure~\ref{fig_digits} presents a representative selection of these images along with their corresponding (binary) labels.

\begin{figure}[h!]
\centering
\includegraphics[scale=0.69]{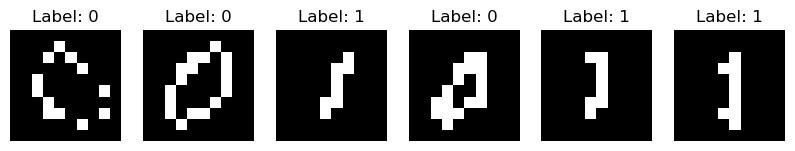}
\caption{Representative MNIST images used in the experimental setup of Example~\ref{ex_mnist}, along with their corresponding (binary) labels.}
\label{fig_digits}
\end{figure}

The architecture of the ANN is as follows:

\begin{itemize}
\item Input Layer: $100$ neurons, each corresponding to one binary pixel input, denoted $X_1,\ldots,X_{100}$.

\item Hidden Layer: A single hidden layer containing $10$ neurons, each using a ReLU activation function, chosen to balance simplicity and sufficient expressive power.

\item Output Layer: A single neuron with a sigmoid activation function, producing a real-valued estimation $\hat{y}\in[0,1]$. The expected output is $\hat{y} = 0$ for images representing the digit $0$, and $\hat{y} = 1$ for images representing the digit $1$. The prediction $\hat{y}$ is thresholded at $\tau = 0.5$, yielding the final binary estimation $\hat{Y}$ via the binary output $Y$.
\end{itemize}

The ANN is compiled using the binary cross-entropy loss function, with network parameters randomly initialized prior to training.

\begin{table}[t]
\centering\scriptsize
{\renewcommand{\arraystretch}{1.5}
\begin{tabular}{|M{1.5cm}||M{1.5cm}|M{1.5cm}|M{1.5cm}|M{1.5cm}|M{1.5cm}||M{0.7cm}|}
 \hline
 \rowcolor{gray!10}
 {\bf Input} & {\bf 1st Output}\, ($\hat{Y}$) & {\bf 2nd Output}\, ($\hat{Y}$) & {\bf 3rd Output} ($\hat{Y}$) & {\bf 4th Output}\, ($\hat{Y}$) & {\bf 5th Output}\, ($\hat{Y}$) & {\bf Label} \\
 \hline\hline
 Image~1 & $1$ & $1$ & $1$ & $1$ & $1$ & {\bf 0} \\
 \hline
 Image~2 & $0$ & $0$ & $0$ & $0$ & $0$ & {\bf 0} \\
 \hline
 Image~3 & $0$ & $0$ & $0$ & $1$ & $1$ & {\bf 1} \\
 \hline
 $\cdots$ & $\cdots$ & $\cdots$ & $\cdots$ & $\cdots$ & $\cdots$ & $\cdots$ \\
 \hline\hline
 {\bf Accuracy} & $43.33\%$ & $56.67\%$ & $70\%$ & $86.67\%$ & $96.67\%$ &--- \\
 \hline
\end{tabular}}
\caption{Successive transitions (estimations) of the binary output $Y$ produced by the ANN during training, evaluated on the first three MNIST images from the training dataset. The ANN's classification accuracy at each training stage is also reported.}
\label{table_mnist}
\end{table}

Within this setting, Table~\ref{table_mnist} presents the successive transitions (estimations) of the binary output $Y$ produced by the ANN during training, evaluated on the first three images from the training dataset. As shown, the network's classification accuracy steadily improves, ultimately reaching an impressive $96.67\%$ on the full training set, indicating that the ANN correctly classified $29$ out of the $30$ samples.

Finally, the experimental analysis confirms that the belief sets associated with the successive outputs of $Y$ evolve in accordance with condition (SD).
\end{example}

\section{Machine Learning is DP-Compatible}
\label{section_ml_dp}

As outlined in Subsection~\ref{subsection_backprop}, the backpropagation algorithm is a fundamental gradient-based method for training feed-forward ANNs. Through iterative adjustment of network's parameters, it aims to reduce the loss function $\mathcal{L}$ to a satisfactory level. Throughout this training process, a single-output binary ANN evolves through a sequence of belief sets, each representing a distinct Boolean function corresponding to the current state of its binary output. Building on this interpretation, Theorem~\ref{thm_ann_full_meet} was formulated in \cite{aravanis25a}, which establishes that the entire progression of belief-set transitions in such a network can be effectively modelled using the full-meet AGM revision function in conjunction with the full-meet AGM contraction function.

\begin{theorem}[\cite{aravanis25a}] \label{thm_ann_full_meet}
Let $Y$ be the binary output of a single-output binary ANN. Let $K_1$, $K_2$ be belief sets that represent, respectively, the Boolean functions of $Y$ right before and right after an arbitrary update of the parameters of the ANN, that is implemented during its training. Moreover, let $\ast$ be the full-meet AGM revision function, and let $\dotminus$ be the full-meet AGM contraction function. There exist two sentences $\varphi_1,\varphi_2\in\mathbb{L}$, such that $K_2 = \big(K_1\ast\varphi_1\big)\dotminus\varphi_2$.
\end{theorem}

While Theorem~\ref{thm_ann_full_meet} highlights that the belief dynamics induced by training a binary ANN can be framed within the classical belief-change framework, using a minimal yet powerful pair of operators, it is important to recognize that full-meet belief change is subject to several well-documented shortcomings, as pointed out in Subsection~\ref{subsection_full_meet} and Section~\ref{section_iteration}. These limitations render it a ``poorly behaved'' belief-change operator in many contexts. In light of these concerns, we formulate Theorem~\ref{thm_backprop_dp} below, which, under a plausible assumption encoded into condition (SD) of Subsection~\ref{subsection_distance_sets}, models the training process of binary ANNs using lexicographic revision and moderate contraction. These belief-change operations are both {\em well-behaved} and {\em fully compatible} with the DP approach.\footnote{It is important to note that, in Theorem~\ref{thm_backprop_dp}, the belief sets \(K_1,\ldots,K_n\) are to be understood as the belief contents of underlying epistemic states; the corresponding faithful preorders are specified and propagated in the proof.}

\begin{theorem} \label{thm_backprop_dp}
Let $Y$ be the binary output of a single-output binary ANN. For $n\geqslant 2$, let $K_1,\ldots,K_n$ be the sequence of belief sets such that, for any \mbox{$i\in\{1,\ldots,n-1\}$}, $K_i$ and $K_{i+1}$ represent, respectively, the Boolean functions of $Y$ right before and right after the $i$-th update of parameters, that is implemented during training. If condition (SD) is satisfied at $K_1,\ldots,K_n$, then there exist a single DP revision function $\ast$, a single DP contraction function $\dotminus$, and two sentences $\varphi_{1}^{i},\varphi_{2}^{i}\in\mathbb{L}$, such that \mbox{$K_{i+1} = \big(K_i\ast\varphi_{1}^{i}\big)\dotminus\varphi_{2}^{i}$}, for any \mbox{$i\in\{1,\ldots,n-1\}$}.
\end{theorem}

\begin{proof}
Assume that condition (SD) is satisfied at $K_1,\ldots,K_n$, meaning that, for any \mbox{$i,j,m\in\{1,\ldots,n\}$} such that $i < j < m$, {\em Dist}$_B\big(K_i,K_m\big) \supseteq$ {\em Dist}$_B\big(K_j,K_m\big)$.

\vspace{3mm}

\noindent\underline{Sentences:} Let $i$ be any index in $\{1,\ldots,n-1\}$, and let $K_i$ and $K_{i+1}$ be the two successive belief sets. Depending on whether $K_i$ and $K_{i+1}$ are mutually inconsistent or not (see Figure~\ref{fig_worlds}), we define two sentences $\varphi_{1}^{i},\varphi_{2}^{i}\in\mathbb{L}$ such that:

\begin{center}
$[\varphi_{1}^{i}] =
\begin{cases}
[K_{i+1}] & \text{if }\enspace [K_i]\cap[K_{i+1}] = \varnothing \\[6pt]
[K_i]\cap[K_{i+1}] & \text{if }\enspace [K_i]\cap[K_{i+1}] \neq \varnothing
\end{cases}$
\end{center}

\vspace{1mm}

\begin{center}
$[\neg\varphi_{2}^{i}] =
\begin{cases}
\varnothing & \text{if }\enspace [K_i]\cap[K_{i+1}] = \varnothing \\[6pt]
[K_{i+1}]\setminus[K_i] & \text{if }\enspace [K_i]\cap[K_{i+1}] \neq \varnothing
\end{cases}$
\end{center}

\begin{figure}[h!]
\centering\footnotesize
\begin{minipage}{0.48\textwidth}
\centering
\begin{tikzpicture}
\node [rectangle, draw, very thick, minimum width=5cm, minimum height=3.5cm, rounded corners] (1) at (0,0) {};
\node [yshift=4mm, align=center] at (1.south) {$\mathbb{M} = [\varphi_{2}^{i}]$};
\draw [thick] (-1.1,0) circle (1cm);
\draw [thick] (1.1,0) circle (1cm);
\node [align=center] at (-1.1,1.3) {$[K_i]$};
\node [align=center] at (1.1,1.3) {$[K_{i+1}]$};
\node[] at (1.1,0) {$\boldsymbol{[\varphi_{1}^{i}]}$};
\end{tikzpicture}
\caption*{$[K_i]\cap[K_{i+1}] = \varnothing$}
\end{minipage}
\hfill
\begin{minipage}{0.48\textwidth}
\centering
\begin{tikzpicture}
\node [rectangle, draw, very thick, minimum width=5cm, minimum height=3.5cm, rounded corners] (2) at (0,0) {};
\node [yshift=4mm, align=center] at (2.south) {$\mathbb{M}$};
\draw [thick] (-0.6,0) circle (1cm);
\draw [thick] (0.6,0) circle (1cm);
\node [align=center] at (-0.7,1.3) {$[K_i]$};
\node [align=center] at (0.7,1.3) {$[K_{i+1}]$};
\node[] at (0,0) {$[\varphi_{1}^{i}]$};
\node[] at (1,0) {$\boldsymbol{[\neg\varphi_{2}^{i}]}$};
\end{tikzpicture}
\caption*{$[K_i]\cap[K_{i+1}] \neq \varnothing$}
\end{minipage}

\vspace{0.5em}
\caption{Sets of possible worlds illustrating whether $[K_i]\cap[K_{i+1}] = \varnothing$ or \mbox{$[K_i]\cap[K_{i+1}] \neq \varnothing$}. Bold highlights indicate $[\varphi_{1}^{i}]$ in the former case, and $[\neg\varphi_{2}^{i}]$ in the latter case, both playing a key role in the proof's subsequent stages.}
\label{fig_worlds}
\end{figure}


Thereafter, we establish two supplementary results concerning key properties of the sentences $\varphi_{1}^{i}$ and $\varphi_{2}^{i}$, which will help us establish the main line of our argument.

\vspace{-0.7cm}

\begin{displayquote}
\begin{lemmaA}
\phantomsection
\label{lemma:A}
Let $\ i\in\{1,\ldots,n-1\}$. \enspace If \enspace $[K_i]\cap[K_{i+1}] = \varnothing$, \enspace then\newline $[\varphi_{1}^{i}] \cap \Big([K_1]\cup\cdots\cup[K_{i}] \Big) = \varnothing$.
\end{lemmaA}
\vspace{-0.4cm}
\begin{proof}
Assume that $[K_i]\cap[K_{i+1}] = \varnothing$. By the specification of the sentence $\varphi_{1}^{i}$, we have that $[\varphi_{1}^{i}]=[K_{i+1}]$. Let $r\in[\varphi_{1}^{i}]$. Hence, $r\in[K_{i+1}]$ and $r\notin[K_{i}]$, meaning that $r\in$ {\em Dist}$_B\big(K_{i},K_{i+1}\big)$. Now, let $j<i$ and suppose towards contradiction that $r\in[K_j]$. Then, \mbox{$r\notin$ {\em Dist}$_B\big(K_{j},K_{i+1}\big)$}, which contradicts \mbox{{\em Dist}$_B\big(K_{j},K_{i+1}\big) \supseteq$ {\em Dist}$_B\big(K_{i},K_{i+1}\big)$}, imposed by condition (SD). Hence, \mbox{$r\notin[K_j]$}. Thus, we have shown that $[\varphi_{1}^{i}] \cap \Big([K_1]\cup\cdots\cup[K_{i}] \Big) = \varnothing$.
\end{proof}
\end{displayquote}

\vspace{-1cm}

\begin{displayquote}
\begin{lemmaB}
\phantomsection
\label{lemma:B}
Let $\ i\in\{1,\ldots,n-1\}$. Then, $[\neg\varphi_{2}^{i}] \cap \Big([K_1]\cup\cdots\cup[K_{i}] \Big) = \varnothing$.
\end{lemmaB}
\vspace{-0.4cm}
\begin{proof}
By the specification of the sentence $\varphi_{2}^{i}$, we have that \mbox{$[\neg\varphi_{2}^{i}] \subseteq [K_{i+1}]\setminus[K_i]$} (either $[K_i]\cap[K_{i+1}] = \varnothing$ or $[K_i]\cap[K_{i+1}] \neq \varnothing$). Let $r\in[\neg\varphi_{2}^{i}]$. Hence, \mbox{$r\in[K_{i+1}]$} and $r\notin[K_{i}]$, meaning that $r\in$ {\em Dist}$_B\big(K_{i},K_{i+1}\big)$. Now, let $j<i$ and suppose towards contradiction that $r\in[K_j]$. Then, \mbox{$r\notin$ {\em Dist}$_B\big(K_{j},K_{i+1}\big)$}, which contradicts \mbox{{\em Dist}$_B\big(K_{j},K_{i+1}\big) \supseteq$ {\em Dist}$_B\big(K_{i},K_{i+1}\big)$}, imposed by condition (SD). Hence, \mbox{$r\notin[K_j]$}. Thus, we have shown that \mbox{$[\neg\varphi_{2}^{i}] \cap \Big([K_1]\cup\cdots\cup[K_{i}] \Big) = \varnothing$}.
\end{proof}
\end{displayquote}


\noindent\underline{Belief-Change Functions:} Let $i\in\{1,\ldots,n-1\}$, let $\ast$ be a lexicographic AGM revision function, and let $\dotminus$ be a moderate AGM contraction function (fixed across all iterations). Suppose that $\ast$ assigns (via (R)) at:

\begin{itemize}
\item $K_1$ a faithful preorder $\preceq_{K_1}^{\ast}$ over $\mathbb{M}$, such that $\min(\mathbb{M},\preceq_{K_1}^{\ast}) = [K_1]$, and for all \mbox{$r,r'\in\mathbb{M}\setminus[K_1]$}, $r \approx_{K_1}^{\ast} r'$.\footnote{Hence, the faithful preorder $\preceq_{K_1}^{\ast}$ attributes equal plausibility (modulo $K_1$) to all possible worlds outside of $[K_1]$.}

\item $K_{i+1}$ a faithful preorder $\preceq_{K_{i+1}}^{\ast}$ over $\mathbb{M}$, such that $\preceq_{K_{i+1}}^{\ast}\ =\ \preceq_{K_{i+1}}^{\dotminus}$.
\end{itemize}

\noindent Moreover, suppose that $\dotminus$ assigns (via (C)) at $K_i\ast\varphi_{1}^{i}$ a faithful preorder $\preceq_{K_i\ast\varphi_{1}^{i}}^{\dotminus}$ over $\mathbb{M}$, such that $\preceq_{K_i\ast\varphi_{1}^{i}}^{\dotminus}\ =\ \preceq_{K_i\ast\varphi_{1}^{i}}^{\ast}$. 

Thus, the operators $\ast$ and $\dotminus$, whose existence is not hard to verify, assign (via (R) and (C)) certain faithful preorders at the belief sets $K_1$, $K_i\ast\varphi_{1}^{i}$, and $K_{i+1}$, as abstractly illustrated in Figure~\ref{fig_preorders}.

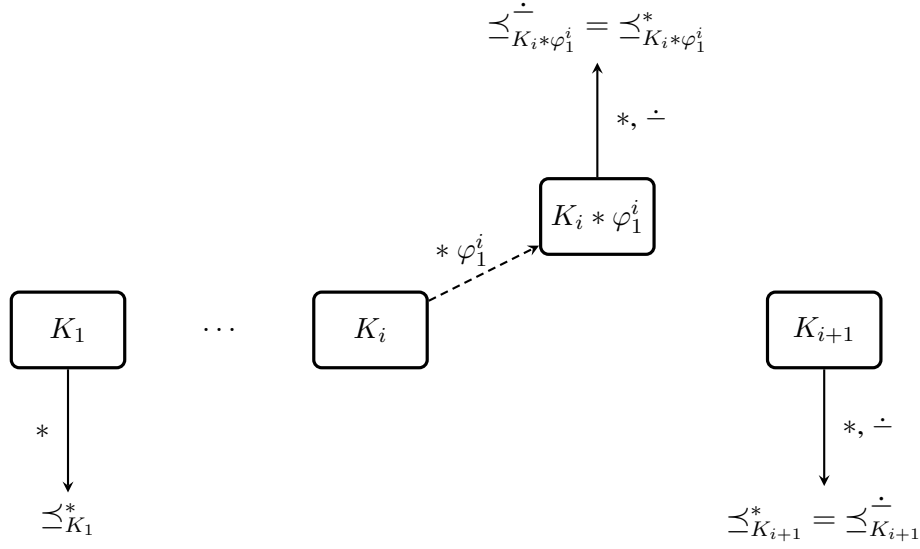
\begin{figure}[h!]
\centering
\begin{tikzpicture}[]
\node (1) [draw, rectangle, rounded corners=3pt, minimum width=1.5cm, minimum height=1cm, align=center, very thick] at (0,0) {$K_1$};
\node (11) [align=center] at (0,-2.5) {\large $\preceq_{K_1}^{\ast}$};
\node (2) [align=center] at (2,0) {$\ldots$};
\node (3) [draw, rectangle, rounded corners=3pt, minimum width=1.5cm, minimum height=1cm, align=center, very thick] at (4,0) {$K_i$};
\node (4) [draw, rectangle, rounded corners=3pt, minimum width=1.5cm, minimum height=1cm, align=center, very thick] at (7,1.5) {$K_i\ast\varphi_{1}^{i}$};
\node (44) [align=center] at (7,4) {\large $\preceq_{K_i\ast\varphi_{1}^{i}}^{\dotminus}\ =\ \preceq_{K_i\ast\varphi_{1}^{i}}^{\ast}$};
\node (5) [draw, rectangle, rounded corners=3pt, minimum width=1.5cm, minimum height=1cm, align=center, very thick] at (10,0) {$K_{i+1}$};
\node (55) [align=center] at (10,-2.5) {\large $\preceq_{K_{i+1}}^{\ast}\ =\ \preceq_{K_{i+1}}^{\dotminus}$};

\draw[->, >=stealth, thick] (1) -- (11) node[midway, left=1mm]{$\ast$};
\draw[->, >=stealth, thick] (4) -- (44) node[midway, right=1mm]{$\ast$, $\dotminus$};
\draw[->, >=stealth, thick] (5) -- (55) node[midway, right=1mm]{$\ast$, $\dotminus$};
\draw[->, >=stealth, thick, densely dashed] (3) -- (4) node[pos=0.3, above=1mm]{$\ast\ \varphi_{1}^{i}$};
\end{tikzpicture}
\caption{Schematic representation of the faithful preorders assigned by $\ast$ and $\dotminus$ at the belief sets $K_1$, $K_i\ast\varphi_{1}^{i}$, and  $K_{i+1}$.}
\label{fig_preorders}
\end{figure}


\vspace{3mm}

\noindent\underline{Induction:} We now proceed by induction. For each $i\in\{1,\ldots,n-1\}$, let $H_i=(K_i\ast\varphi_1^i)\dotminus\varphi_2^i$. We shall prove, by induction on $i$, the following invariant:
\[
(\mathcal{I}_i)\qquad \text{For all } r,r'\in\mathbb{M}\setminus\Big([K_1]\cup\cdots\cup[K_i]\Big), \quad r\approx^\ast_{K_i}r'.
\]

The base case $(\mathcal{I}_1)$ follows immediately from the construction of $\preceq^\ast_{K_1}$, according to which all worlds outside $[K_1]$ are equally plausible with respect to $\preceq^\ast_{K_1}$.

Assume now that $(\mathcal{I}_i)$ holds, for some $i\in\{1,\ldots,n-1\}$. We first show that $[H_i]=[K_{i+1}]$. There are two cases.

\begin{itemize}
\item Assume that $[K_i]\cap[K_{i+1}]=\varnothing$. Then, by the definition of $\varphi_1^i$ and $\varphi_2^i$, $[\varphi_1^i]=[K_{i+1}]$ and $[\neg\varphi_2^i]=\varnothing$. Hence, $\varphi_2^i\equiv\top$. By Lemma~\hyperref[lemma:A]{A}, $[\varphi_1^i]\cap\Big([K_1]\cup\cdots\cup[K_i]\Big)=\varnothing$. Therefore, for any worlds $r,r'\in[\varphi_1^i]$, the induction hypothesis $(\mathcal{I}_i)$ entails that $r\approx^\ast_{K_i}r'$.  Consequently, by condition (R), it follows that $[K_i\ast\varphi_1^i] = \min([\varphi_1^i],\preceq^\ast_{K_i}) = [\varphi_1^i] = [K_{i+1}]$.

Since $\varphi_2^i\equiv\top$, contraction by $\varphi_2^i$ is trivial (Remark~\ref{rem_tautology}). Hence, $[H_i] = \big[(K_i\ast\varphi_1^i)\dotminus\varphi_2^i\big] = [K_i\ast\varphi_1^i] = [K_{i+1}]$.

\item Assume that $[K_i]\cap[K_{i+1}]\neq\varnothing$. Then, by the definition of $\varphi_1^i$ and $\varphi_2^i$, $[\varphi_1^i]=[K_i]\cap[K_{i+1}]$ and $[\neg\varphi_2^i]=[K_{i+1}]\setminus[K_i]$. Since $[\varphi_1^i]\subseteq[K_i]$, the sentence $\varphi_1^i$ is consistent with $K_i$. Therefore, by postulates $(K\ast3)$ and $(K\ast4)$, revision by $\varphi_1^i$ reduces to expansion; i.e., $K_i\ast\varphi_1^i=K_i+\varphi_1^i$. Hence, $[K_i\ast\varphi_1^i] = [K_i]\cap[\varphi_1^i] = [K_i]\cap[K_{i+1}]$.

We next show that all worlds in $[\neg\varphi_2^i]$ are equally plausible with respect to $\preceq^{\dotminus}_{K_i\ast\varphi_1^i}$. Let $r,r'\in[\neg\varphi_2^i]$. By Lemma~\hyperref[lemma:B]{B}, $[\neg\varphi_2^i]\cap\Big([K_1]\cup\cdots\cup[K_i]\Big)=\varnothing$. Hence, \mbox{$r,r'\in \mathbb{M}\setminus\Big([K_1]\cup\cdots\cup[K_i]\Big)$}. By the induction hypothesis $(\mathcal{I}_i)$, it follows that $r\approx^\ast_{K_i}r'$. Moreover, since $[\neg\varphi_2^i]=[K_{i+1}]\setminus[K_i]$ and $[\varphi_1^i]=[K_i]\cap[K_{i+1}]$, we have that $[\neg\varphi_2^i]\subseteq[\neg\varphi_1^i]$. Thus, $r,r'\in[\neg\varphi_1^i]$. Since $\ast$ is a lexicographic AGM revision function, the transition from $\preceq^\ast_{K_i}$ to $\preceq^\ast_{K_i\ast\varphi_1^i}$ satisfies condition (R2). Therefore, $r\approx^\ast_{K_i\ast\varphi_1^i}r'$. By the definition of $\ast$ and $\dotminus$, we derive that \mbox{$\preceq^{\dotminus}_{K_i\ast\varphi_1^i}\ =\ \preceq^\ast_{K_i\ast\varphi_1^i}$}. Consequently, $r\approx^{\dotminus}_{K_i\ast\varphi_1^i}r'$. Since $r,r'$ were arbitrary worlds in $[\neg\varphi_2^i]$, we obtain $\min\Big([\neg\varphi_2^i], \preceq^{\dotminus}_{K_i\ast\varphi_1^i}\Big) = [\neg\varphi_2^i]$. By condition (C), we now have that $[H_i] = \big[(K_i\ast\varphi_1^i)\dotminus\varphi_2^i\big] = [K_i\ast\varphi_1^i] \cup \min\Big([\neg\varphi_2^i], \preceq^{\dotminus}_{K_i\ast\varphi_1^i}\Big)$. Therefore, $[H_i] = \big([K_i]\cap[K_{i+1}]\big) \cup \big([K_{i+1}]\setminus[K_i]\big) = [K_{i+1}]$.
\end{itemize}

Thus, in both cases, $[H_i]=[K_{i+1}]$. It remains to show that the invariant is preserved. Let $r,r'\in \mathbb{M}\setminus\Big([K_1]\cup\cdots\cup[K_i]\cup[K_{i+1}]\Big)$. Then, in particular, $r,r'\in \mathbb{M}\setminus\Big([K_1]\cup\cdots\cup[K_i]\Big)$. By the induction hypothesis $(\mathcal{I}_i)$, $r\approx^\ast_{K_i}r'$. Moreover, $r,r'\in\mathbb{M}\setminus\Big([K_i]\cup[K_{i+1}]\Big)$. By the definition of $\varphi_1^i$, this entails that $r,r'\in[\neg\varphi_1^i]$. Since $\ast$ is a lexicographic AGM revision function, condition (R2) entails that the ordering among $\neg\varphi_1^i$-worlds is preserved when passing from $\preceq^\ast_{K_i}$ to $\preceq^\ast_{K_i\ast\varphi_1^i}$. Hence, $r\approx^\ast_{K_i\ast\varphi_1^i}r'$. By the definition of $\ast$ and $\dotminus$, $\preceq^{\dotminus}_{K_i\ast\varphi_1^i}\ =\ \preceq^\ast_{K_i\ast\varphi_1^i}$. Thus, $r\approx^{\dotminus}_{K_i\ast\varphi_1^i}r'$.

Furthermore, by the definition of $\varphi_2^i$, every world outside $[K_i]\cup[K_{i+1}]$ is a $\varphi_2^i$-world. Hence, $r,r'\in[\varphi_2^i]$. Since $\dotminus$ is a moderate AGM contraction function, the transition from $\preceq^{\dotminus}_{K_i\ast\varphi_1^i}$ to $\preceq^{\dotminus}_{H_i}$ satisfies condition (C1). Therefore, $r\approx^{\dotminus}_{H_i}r'$. Since we have already shown that $[H_i]=[K_{i+1}]$, and since the preorder assigned to $K_{i+1}$ is inherited posterior preorder of this contraction step, namely $\preceq^\ast_{K_{i+1}}\ =\ \preceq^{\dotminus}_{H_i}$, we obtain that $r\approx^\ast_{K_{i+1}}r'$. Thus, $(\mathcal{I}_{i+1})$ holds.

By induction, $(\mathcal{I}_i)$ holds for every $i\in\{1,\ldots,n\}$. Moreover, for every $i\in\{1,\ldots,n-1\}$, we have shown that $\big[(K_i\ast\varphi_1^i)\dotminus\varphi_2^i\big]=[K_{i+1}]$. Therefore, $K_{i+1} = (K_i\ast\varphi_1^i)\dotminus\varphi_2^i$, for every $i\in\{1,\ldots,n-1\}$.


\vspace{3mm}

Consequently, and since $\ast$ and $\dotminus$ satisfy postulates (DP1)--(DP4) and (DPC1)--(DPC4), respectively (Remark~\ref{rem_subclass_dp} of Section~\ref{section_iteration}), we have shown that there exist a DP revision function $\ast$, a DP contraction function $\dotminus$, and two sentences $\varphi_{1}^{i},\varphi_{2}^{i}\in\mathbb{L}$, such that \mbox{$K_{i+1} = \big(K_i\ast\varphi_{1}^{i}\big)\dotminus\varphi_{2}^{i}$}, for any \mbox{$i\in\{1,\ldots,n-1\}$}, as desired.
\end{proof}

Based on the proof of Theorem~\ref{thm_backprop_dp}, we comment on the interpretation of the epistemic inputs $\varphi_{1}^{i}$ and $\varphi_{2}^{i}$ that mediate the transition between successive belief sets during the training process of a binary ANN. These sentences are not arbitrary, but are systematically constructed to reflect the symbolic effect of each parameter update on the network's Boolean behaviour. Since the evolution from $K_i$ to $K_{i+1}$ is driven by the learning process ---which in turn is governed by the training labels--- it follows that $\varphi_{1}^{i}$ and $\varphi_{2}^{i}$ encode propositional information induced by those labels.\footnote{Very roughly speaking, $\varphi_{1}^{i}$ identifies the minimal content that must be revised into $K_i$ to initiate alignment with $K_{i+1}$, while $\varphi_{2}^{i}$ isolates the residual content that must be contracted to fully obtain $K_{i+1}$.} 

Furthermore, if $T$ denotes the belief set corresponding to the target Boolean function implicitly defined by the training labels, and condition (SD) is respected at the sequence $K_1, \ldots, K_n$, then this sequence follows an inclusion-decreasing trend in terms of type-B distance from $T$. Thus, $K_{i+1}$ is {\em epistemically closer} to $T$ than $K_i$, and the transition $(K_i \ast \varphi_{1}^{i}) \dotminus \varphi_{2}^{i}$ can be seen as a symbolic operation that draws the network's knowledge state closer to the label-induced target knowledge. In this light, $\varphi_{1}^{i}$ and $\varphi_{2}^{i}$ act as epistemic surrogates of the training signal at each iteration, representing label-driven logical corrections to the ANN's internal representation. This characterization invites further study into the precise syntactic and semantic relationship between the ANN's training data and the propositional forms of $\varphi_{1}^{i}, \varphi_{2}^{i}$, opening new avenues for a fine-grained symbolic analysis of learning dynamics in binary ANNs.

Example~\ref{ex_backprop_dp} offers a concrete illustration of how the training process of a binary ANN can be modelled using DP-compatible AGM-style operations --- namely, lexicographic revision and moderate contraction. This is achieved by instantiating the construction mechanism outlined in the proof of Theorem~\ref{thm_backprop_dp}.

\begin{example}[DP-Compatible Belief Change in a Binary ANN] \label{ex_backprop_dp}
Let $K_1$, $K_2$, $K_3$ be the belief sets that represent the three distinct Boolean functions to which the binary output $Y$ of a binary ANN successively adheres during training. Let $\{r_1,r_2,r_3,r_4,r_5\}\subseteq\mathbb{M}$ be a set of possible worlds, and suppose that $[K_1] = \{r_1\}$, $[K_2] = \{r_4,r_5\}$, and $[K_3] = \{r_2,r_3,r_4,r_5\}$. The transitions of belief sets that represent the evolving Boolean functions of $Y$ are depicted in Figure~\ref{fig_ex_transitions}, where a value of $1$ indicates that the corresponding world is satisfied at a belief set, and $0$ otherwise. Observe that $\{r_1,r_2,r_3,r_4,r_5\} =$ Dist$_{B}\big(K_1,K_3\big) \supset\, $ Dist$_{B}\big(K_2,K_3\big) = \{r_2,r_3\}$, meaning that condition (SD) of Section~\ref{section_change_gradual} is satisfied at the sequence $K_1$, $K_2$, $K_3$.

\begin{figure}[h!]
\centering
\begin{tikzpicture}
\node at (-1.5,0)  (r1) {$\bf r_1$};
\node at (-1.5,-1) (r2) {$\bf r_2$};
\node at (-1.5,-2) (r3) {$\bf r_3$};
\node at (-1.5,-3) (r4) {$\bf r_4$};
\node at (-1.5,-4) (r5) {$\bf r_5$};
\node at (0,0)  (a1) {1};
\node at (0,-1) (a2) {0};
\node at (0,-2) (a3) {0};
\node at (0,-3) (a4) {0};
\node at (0,-4) (a5) {0};
\node at (2,0)  (b1) {0};
\node at (2,-1) (b2) {0};
\node at (2,-2) (b3) {0};
\node at (2,-3) (b4) {1};
\node at (2,-4) (b5) {1};
\node at (4,0)  (c1) {0};
\node at (4,-1) (c2) {1};
\node at (4,-2) (c3) {1};
\node at (4,-3) (c4) {1};
\node at (4,-4) (c5) {1};
\foreach \i in {2,3} {
\draw[->,>=stealth,shorten >=6pt,shorten <=6pt,thick,dashed] (a\i) -- (b\i); }
\draw[->,>=stealth,shorten >=6pt,shorten <=6pt,thick] (a1) -- (b1);
\draw[->,>=stealth,shorten >=6pt,shorten <=6pt,thick] (a4) -- (b4);
\draw[->,>=stealth,shorten >=6pt,shorten <=6pt,thick] (a5) -- (b5);

\foreach \i in {1,2,4,5} {
\draw[->,>=stealth,shorten >=6pt,shorten <=6pt,thick,dashed] (b\i) -- (c\i); }
\draw[->,>=stealth,shorten >=6pt,shorten <=6pt,thick] (b2) -- (c2);
\draw[->,>=stealth,shorten >=6pt,shorten <=6pt,thick] (b3) -- (c3);

\node at (0,1) {$\bf K_1$};
\node at (2,1) {$\bf K_2$};
\node at (4,1) {$\bf K_3$};
\end{tikzpicture}
\caption{Successive belief-set transitions representing the evolving Boolean functions of $Y$ during training. Solid arrows highlight changes in the satisfiability of worlds across adjacent belief sets.}
\label{fig_ex_transitions}
\end{figure}

Let $\ast$ be a lexicographic AGM revision function, and let $\dotminus$ be a moderate AGM contraction function. Suppose that $\ast$ assigns (via (R)) at $K_1$ a faithful preorder $\preceq_{K_1}^{\ast}$ over $\mathbb{M}$, such that

\begin{center}
$\underbrace{r_1}_{[K_1]} \quad \prec_{K_1}^{\ast}\quad r_2\quad \approx_{K_1}^{\ast}\quad r_3\quad \approx_{K_1}^{\ast}\quad r_4\quad \approx_{K_1}^{\ast}\quad r_5$.
\end{center}

\noindent Moreover, suppose that $\ast$ assigns (via (R)) at $K_2$ a faithful preorder $\preceq_{K_2}^{\ast}$, such that $\preceq_{K_2}^{\ast}\ =\ \preceq_{K_2}^{\dotminus}$. Finally, suppose that, for any \mbox{$i\in\{1,2\}$}, $\dotminus$ assigns (via (C)) at $K_i\ast\varphi_{1}^{i}$ a faithful preorder $\preceq_{K_i\ast\varphi_{1}^{i}}^{\dotminus}$ over $\mathbb{M}$, such that $\preceq_{K_i\ast\varphi_{1}^{i}}^{\dotminus}\ =\ \preceq_{K_i\ast\varphi_{1}^{i}}^{\ast}$.

Thereafter, we examine separately the two transitions $K_1$-to-$K_2$ and $K_2$-to-$K_3$.

\begin{itemize}
\item \underline{From $K_1$ to $K_2$} 

\noindent First, observe that $[K_1]\cap[K_2] = \varnothing$. Let $\varphi_{1}^{1}$, $\varphi_{2}^{1}$ be two sentences of $\mathbb{L}$ such that \mbox{$[\varphi_{1}^{1}] = [K_2] = \{r_4,r_5\}$} and let $[\neg\varphi_{2}^{1}]=\varnothing$ (i.e., $\varphi_{2}^{1} \equiv\top$). Then, we derive from condition (R) that $[K_1\ast\varphi_{1}^{1}] = \min([\varphi_{1}^{1}],\preceq_{K_1}^{\ast}) = [\varphi_{1}^{1}] = \{r_4,r_5\} = [K_2]$. Moreover, since $\ast$ is a lexicographic AGM revision function, it follows that the faithful preorder $\preceq_{K_1\ast\varphi_{1}^{1}}^{\ast}$ is such that

\begin{center}
$\underbrace{r_4\enspace \approx_{K_1\ast\varphi_{1}^{1}}^{\ast}\enspace r_5}_{[K_1\ast\varphi_{1}^{1}]} \quad \prec_{K_1\ast\varphi_{1}^{1}}^{\ast}\quad r_1\quad \prec_{K_1\ast\varphi_{1}^{1}}^{\ast}\quad r_2\quad \approx_{K_1\ast\varphi_{1}^{1}}^{\ast}\quad r_3$.
\end{center}

\noindent Given that $\preceq_{K_1\ast\varphi_{1}^{1}}^{\dotminus}\ =\ \preceq_{K_1\ast\varphi_{1}^{1}}^{\ast}$ (by the definition of $\ast$ and $\dotminus$), we derive from condition (C) that \mbox{$\big[\big(K_1\ast\varphi_{1}^{1}\big)\dotminus\varphi_{2}^{1}\big] =$} \mbox{$[K_1\ast\varphi_{1}^{1}] \cup \min([\neg\varphi_{2}^{1}],\preceq_{K_1\ast\varphi_{1}^{1}}^{\dotminus}) = [K_1\ast\varphi_{1}^{1}] = [K_2]$}. Therefore, \mbox{$K_{2} = \big(K_1\ast\varphi_{1}^{1}\big)\dotminus\varphi_{2}^{1}$}. Furthermore, since $\dotminus$ is a moderate AGM contraction function and $\varphi_{2}^{1} \equiv\top$ (i.e., $[\neg\varphi_{2}^{1}]=\varnothing$), it follows that $\preceq_{K_{2}}^{\dotminus}\ =\ \preceq_{K_1\ast\varphi_{1}^{1}}^{\dotminus}$; namely,

\begin{center}
$\underbrace{r_4\enspace \approx_{K_{2}}^{\dotminus}\enspace r_5}_{[K_2]} \quad \prec_{K_{2}}^{\dotminus}\quad r_1\quad \prec_{K_{2}}^{\dotminus}\quad r_2\quad \approx_{K_{2}}^{\dotminus}\quad r_3$.
\end{center}

\item \underline{From $K_2$ to $K_3$} 

\noindent First, observe that $[K_2]\cap[K_3] \neq \varnothing$. Let $\varphi_{1}^{2}$, $\varphi_{2}^{2}$ be two sentences of $\mathbb{L}$ such that \mbox{$[\varphi_{1}^{2}] = [K_2]\cap[K_3] = \{r_4,r_5\}$} and \mbox{$[\neg\varphi_{2}^{2}] = [K_3]\setminus[K_2] = \{r_2,r_3\}$}. 

Given that $\preceq_{K_2}^{\ast}\ =\ \preceq_{K_2}^{\dotminus}$ (by the definition of $\ast$ and $\dotminus$), we derive from condition (R) that $[K_2\ast\varphi_{1}^{2}] = \min([\varphi_{1}^{2}],\preceq_{K_2}^{\ast}) = [K_2+\varphi_{1}^{2}] = [\varphi_{1}^{2}] = \{r_4,r_5\}$. Moreover, since $\ast$ is a lexicographic AGM revision function, it follows that the faithful preorder $\preceq_{K_2\ast\varphi_{1}^{2}}^{\ast}$ is such that

\begin{center}
$\underbrace{r_4\enspace \approx_{K_2\ast\varphi_{1}^{2}}^{\ast}\enspace r_5}_{[K_2\ast\varphi_{1}^{2}]} \quad \prec_{K_2\ast\varphi_{1}^{2}}^{\ast}\quad r_1\quad \prec_{K_2\ast\varphi_{1}^{2}}^{\ast}\quad r_2\quad \approx_{K_2\ast\varphi_{1}^{2}}^{\ast}\quad r_3$.
\end{center}

\noindent Next, since $\preceq_{K_2\ast\varphi_{1}^{2}}^{\dotminus}\ =\ \preceq_{K_2\ast\varphi_{1}^{2}}^{\ast}$ (by the definition of $\ast$ and $\dotminus$), we derive from condition (C) that \mbox{$\big[\big(K_2\ast\varphi_{1}^{2}\big)\dotminus\varphi_{2}^{2}\big] =$} \mbox{$[K_2\ast\varphi_{1}^{2}] \cup \min([\neg\varphi_{2}^{2}],\preceq_{K_2\ast\varphi_{1}^{2}}^{\dotminus}) = [K_2\ast\varphi_{1}^{2}] \cup [\neg\varphi_{2}^{2}] = \{r_2,r_3,r_4,r_5\} = [K_3]$}. Therefore, \mbox{$K_{3} = \big(K_2\ast\varphi_{1}^{2}\big)\dotminus\varphi_{2}^{2}$}.
\end{itemize}

Consequently, and as expected due to Theorem~\ref{thm_backprop_dp} (since condition (SD) holds), there exist a DP revision function $\ast$, a DP contraction function $\dotminus$, and two sentences $\varphi_{1}^{i},\varphi_{2}^{i}\in\mathbb{L}$, such that \mbox{$K_{i+1} = \big(K_i\ast\varphi_{1}^{i}\big)\dotminus\varphi_{2}^{i}$}, for any \mbox{$i\in\{1,2\}$}.
\end{example}

\section{Conclusion}

This article is a continuation of the research initiated in our earlier work \cite{aravanis25a}, where the foundational connection between belief-change theory and the training dynamics of binary Artificial Neural Networks (ANNs) was introduced. In that work, we proposed a symbolic framework in which the knowledge encoded by binary ANNs is represented via propositional logic, and the training process is interpreted as a sequence of belief-set transitions, modelled through full-meet belief change.

In the present study, we extended that framework by addressing key conceptual and technical limitations. First, we showed that Dalal's method for belief change \cite{dalal88} provides a natural basis for constructing structured and gradual modifications of belief sets through intermediate states of belief. More significantly, we replaced the restrictive full-meet operation with well-behaved belief-change operators ---namely, lexicographic revision \cite{nayak94,nayak03} and moderate contraction \cite{ramachandran12}--- that are consistent with the Darwiche-Pearl approach for iterated belief change \cite{darwiche97}. This refinement allows for a more expressive and robust modelling of the dynamics underlying binary ANN training. More broadly, the present findings resonate with learning-theoretic studies of truth-tracking by belief revision, as well as with recent logical frameworks for learning dynamics, which together suggest that the choice of revision policy affects not only rationality in the AGM/DP sense, but also long-run learnability and convergence to truth \cite{kelly98,baltag19,baltag11,baltag19b,baccini25}.

Taken together, our results strengthen the theoretical foundations of symbolic reasoning in machine learning, and pave the way for future work that explores the interpretability of ANN behaviour through the lens of formal logic and belief dynamics.

\currentpdfbookmark{Acknowledgments}{Acknowledgments}
\section*{Acknowledgments} 

The author expresses appreciation to the anonymous reviewers for their thoughtful suggestions on an earlier version of this article.

\currentpdfbookmark{Declarations}{Declarations}
\section*{Declarations}

\noindent {\bf Competing Interests:} The author has no relevant financial or non-financial interests to disclose.

\vspace{3mm}

\noindent {\bf Data Availability:} Data sharing not applicable to this article as no 

\currentpdfbookmark{References}{References}\bibliographystyle{plain}
\bibliography{references}
\end{document}